\documentclass[sn-mathphys-num]{sn-jnl}
\usepackage{renstyles}
\usepackage{tikz}
\usepackage{enumitem}
\usetikzlibrary{arrows.meta}
\usetikzlibrary{arrows}
\usetikzlibrary{calc}
\usepackage{xcolor}

\begin{document}

\title[Sampling with Adaptive Variance for Multimodal Distributions]{Sampling with Adaptive Variance for Multimodal Distributions}

\author[1]{\fnm{Bj\"orn} \sur{Engquist}}\email{engquist@oden.utexas.edu}
\author[2]{\fnm{Kui} \sur{Ren}}\email{kr2002@columbia.edu}
\author*[3]{\fnm{Yunan} \sur{Yang}}\email{yunan.yang@cornell.edu}

\affil[1]{\orgdiv{Department of Mathematics and the Oden Institute}, \orgname{The University of Texas at Austin}, \orgaddress{\street{2515 Speedway}, \city{Austin}, \postcode{78712}, \state{TX}, \country{USA}}}
\affil[2]{\orgdiv{Department of Applied Physics and Applied Mathematics}, \orgname{Columbia University}, \orgaddress{\street{500 W 120th St}, \city{New York}, \postcode{10027}, \state{NY}, \country{USA}}}
\affil*[3]{\orgdiv{Department of Mathematics}, \orgname{Cornell University}, \orgaddress{\street{212 Garden Ave}, \city{Ithaca}, \postcode{14850}, \state{NY}, \country{USA}}}

\abstract{We propose and analyze a class of adaptive sampling algorithms for multimodal distributions on a bounded domain, which share a structural resemblance to the classic overdamped Langevin dynamics. We first demonstrate that this class of linear dynamics with adaptive diffusion coefficients and vector fields can be interpreted and analyzed as weighted Wasserstein gradient flows of the Kullback--Leibler (KL) divergence between the current distribution and the target Gibbs distribution, which directly leads to the exponential convergence of both the KL and $\chi^2$ divergences, with rates depending on the weighted Wasserstein metric and the Gibbs potential. We then show that a derivative-free version of the dynamics can be used for sampling without gradient information of the Gibbs potential and that for Gibbs distributions with nonconvex potentials, this approach could achieve significantly faster convergence than the classical overdamped Langevin dynamics. A comparison of the mean transition times between local minima of a nonconvex potential further highlights the better efficiency of the derivative-free dynamics in sampling.
}


\maketitle

\section{Introduction}

Sampling algorithms are fundamental tools in statistics, machine learning, and various scientific disciplines. They enable us to draw samples from complex probability distributions, crucial for tasks ranging from Bayesian inference~\cite{welling2011bayesian} to generative modeling~\cite{song2021scorebased}.  Methods, such as Markov Chain Monte Carlo (MCMC)~\cite{brooks2011handbook}, are widely used for this purpose. In recent years, a significant improvement has occurred with the introduction of techniques rooted in optimal transport theory~\cite{Villani-Book03}, particularly those leveraging the Wasserstein geometry~\cite{jordan1998variational}, to analyze the non-asymptotic convergence behavior of various sampling algorithms.  The Wasserstein distance, a measure of discrepancy between probability distributions, provides a powerful framework for analyzing and optimizing the convergence properties of sampling algorithms, such as the Langevin Monte Carlo algorithm~\cite{dalalyan2017further}. This perspective has led to the development of novel algorithms that are more efficient and robust in high-dimensional settings~\cite{chen2020fast,chewi2021optimal,chen2022optimal,altschuler2023resolving}; see~\cite{chewi2024} for a thorough overview of the field.

The relationship between optimization and sampling has revealed deep connections that have advanced both fields~\cite{chewi2023optimization}. Central to this relationship is the Langevin diffusion, which bridges sampling and optimization by adding a stochastic component to the traditional gradient flow used for optimization. This stochasticity enables the algorithm to explore the entire optimization landscape, which then makes it possible to be used as a method for sampling from a distribution. The seminal work of Jordan, Kinderlehrer, and Otto (JKO)~\cite{jordan1998variational} highlighted that the Langevin diffusion could be interpreted as a gradient flow in the space of probability measures equipped with Wasserstein geometry, where the evolution of the marginal law aligns with the gradient flow of the Kullback--Leibler (KL) divergence. This perspective has inspired novel algorithms and facilitated the analysis of optimization techniques such as the proximal point method~\cite{chen2022improved}, Nesterov’s accelerated gradient method adapted to new sampling algorithms~\cite{chen2019hessian,ma2021there}, and many other variants~\cite{wang2022accelerated,chen2023accelerating}.  

However, current sampling algorithms, such as Langevin Monte Carlo, Hamiltonian Monte Carlo, Metropolis--Hastings adjusted sampling methods, face several notable challenges when dealing with non-log-concave distributions~\cite{lee2018beyond,chewi2024}. One primary limitation is the general infeasibility of polynomial-time guarantees for sampling from non-log-concave distributions in common metrics such as the Wasserstein metric~\cite{balasubramanian2022towards}. This contrasts sharply with log-concave distributions, where polynomial-time algorithms are more readily available across various metrics and divergences. For certain non-log-concave distributions, particularly those with multiple modes (e.g., Gaussian mixture models), the convergence of algorithms like Langevin Monte Carlo can be quite slow. These algorithms may require an exponentially long time to transition between modes, resulting in poor mixing. Additionally, the performance of algorithms for non-log-concave sampling is highly sensitive to the choice of initialization and step size. Improper choices can lead to sub-optimal performance or even failure to converge.

Current sampling algorithms often do not have  adaptive diffusion mechanisms. Traditional MCMC methods~\cite{dalalyan2017further,chen2020fast} employ fixed diffusion processes that do not adjust based on the geometry or structure of the target distribution. In simulated tempering~\cite{marinari1992simulated}, analogous to the simulated annealing~\cite{kirkpatrick1983optimization} in global optimization, the diffusion coefficient varies but only depending on time. This rigidity can lead to poor exploration and slow convergence in the presence of multimodal and non-log-concave distributions. As the field progresses, developing algorithms with adaptive diffusion capabilities is crucial for enhancing the performance and applicability of sampling methods.

In~\cite{engquist2022algebraically, engquist2024adaptive}, the authors investigated the benefits of a state-dependent diffusion coefficient for global optimization within the framework of both a stochastic gradient descent algorithm and a derivative-free algorithm, respectively. We proved asymptotic algebraic convergence of the proposed stochastic algorithms in terms of both the distance to the global minimizer and probability.  There are also other recent activities exploring the adaptivity of diffusion coefficient on the state for global optimization; see, for example ,~\cite {wojtowytsch2023stochastic}. While these works address the problem of finding the global minimizer of a highly nonconvex objective function, the encouraging analytical results motivate us to explore the advantages of state-dependent adaptive diffusion in sampling, particularly for non-log-concave distributions.

This work studies the Langevin Monte Carlo sampling method with a state-dependent adaptive diffusion term. The goal is to sample the Gibbs distribution
\begin{equation}\label{EQ:Gibbs}
    \pi_G(\bx)=Z_G^{-1} \exp\left(-\frac{F(\bx)}{\eps}\right),\quad Z_G:=\int_{\bbT^d} \exp\left(-\frac{F(\bx)}{\eps}\right)\, d\bx\,,
\end{equation}
where $F:\bbT^d \rightarrow \mathbb{R}$ is the Gibbs potential function, $\bbT^d$ is a $d$-dimensional torus, and $\eps$ is the standard scaling parameter controlling the flatness of the distribution. A general format of the  algorithm under study is
\begin{equation}\label{EQ:AdVar Disc}
    X_{n+1} = X_n -\eta_n\nabla H(X_n) + \sqrt{2\eta_n\, D(X_n)}\, \xi_n,
\end{equation}
where $X_n \in \bbT^d$ is  the current $n$-th step iterate, $H = h\circ F$ and $D = \sigma \circ F$ with some $h:\mathbb{R} \rightarrow \bbR$ and $\sigma:\mathbb{R} \rightarrow \bbR^+$, $\eta_n>0$ is the step size, and $\xi_n$ is a standard normal random variable.  

The main feature of the proposed iterative algorithm is the $F$-dependent diffusion coefficient, and the resulting different potential function $H:=h\circ F$ instead of $F$. This iteration can be thought of as the Euler--Maruyama discretization of the continuous overdamped Langevin dynamics with an $F$-dependent noise:
\begin{equation}\label{EQ:AdVar Cont}
dX_t = -\nabla H(X_t) dt+\sqrt{2D(X_t)} dW_t\,,
\end{equation}
where $dW_t$ is the standard Brownian motion. In the rest of the work, we will show that by selecting different $F$-dependent diffusion coefficients (and adjusting the potential $H$ accordingly), we can construct a class of algorithms that converges to the same Gibbs distribution with different convergence behaviors.

There have been extensive studies on the sampling algorithm~\eqref{EQ:AdVar Disc} and its continuous limit~\eqref{EQ:AdVar Cont} when $\sigma$ is a constant; see, for instance, \cite{dalalyan2017further,dalalyan2012sparse,DM19,ChChBaJo-PMLR18,VeWi-NIPS19}, for some classical and recent literature. Let $\pi_t$ be the probability law of the process $X_t$ at time $t$ and assume that a unique stationary distribution $\pi$ exists for the system (which often requires that the potential $F$ is such that the stationary distribution $\pi$ satisfies functional inequalities such as the Poincar\'{e} or the log-Sobolev inequalities). Then, a standard convergence result often takes the following form:
\begin{equation}\label{eq: expo converge}
    d(\pi_t, \pi) \le C e^{-\lambda t} d(\pi_0, \pi),\quad \lambda ,C>0\,,
\end{equation}
where $\pi_0$ is the initial distribution. Such results hold for some different choices of distances and divergences (see, for instance, \cite{LeNiPa-JSP13,Pavliotis-Book14,ChErLiShZh-FCM24,BaGeLe-Book24} and references therein) and can be used to estimate the computational complexity of the corresponding discrete sampling algorithm~\cite{ErHo-PMLR21}.

The SDE given in~\eqref{EQ:AdVar Cont} can be studied under the Eulerian framework based on the PDE analysis.
Let  $D(\bx):=\sigma (F(\bx))$ be the diffusion coefficient function, and $H(\bx):=h(F(\bx))$ be the ``effective'' potential generating the drift term. The infinitesimal generator of the semigroup $\cL$ is given as
\[
    \cL f :=-\nabla H \cdot \nabla f+ D \Delta f ,\ \ \ \forall f\in C^2(\bbT^d)\,.
\]
The Fokker--Planck equation corresponding to~\eqref{EQ:AdVar Cont} is
\begin{equation}\label{EQ:AdVar FPE}
    \partial_t \rho = \cL^* \rho:=\nabla\cdot (\rho \nabla H)+ \Delta (D \rho\big)=\nabla\cdot \Big(\rho \nabla H+ \nabla(D \rho)\Big)\,.
\end{equation}
For example, when $(h,\sigma)=(I,\eps)$, or equivalently, $(H(\bx),D(\bx))=(F(\bx),\eps)$, the Gibbs measure~\eqref{EQ:Gibbs} is a steady-state solution to~\eqref{EQ:AdVar FPE} on $\bbR^d$~\cite{HoKuSt-JFA89,Pavliotis-Book14}. 

Among the entire class of dynamics in the form of~\eqref{EQ:AdVar Cont} or \eqref{EQ:AdVar FPE},  we will pay special attention to the derivative-free case in which the drift term is zero (i.e. $H$ is a constant). In this scenario,  the stationary distribution of the problem is
\begin{equation*}
    \pi(\bx) = Z^{-1} \frac{1}{D(\bx)}\,,
\end{equation*}
assuming again that $D^{-1}$ is integrable on $\bbT^d$ and $Z = \int_{\bbT^d} D^{-1}(\bx)dx$ . One key point we aim to address in this work is that by selecting
\begin{equation}\label{EQ:Sigma Gibbs} 
    D(\bx) = \sigma(F(\bx)) \varpropto \exp\left(\frac{F(\bx)}{\eps}\right)
\end{equation}
in the derivative-free case, we recover the Gibbs measure introduced in~\eqref{EQ:Gibbs}. This observation suggests that one could potentially sample the Gibbs measure in a derivative-free manner using the adaptive variance scheme with $\sigma$ depending on $F$ in the form of~\eqref{EQ:Sigma Gibbs}.  This approach is particularly advantageous for large-scale applications where gradient evaluation is either infeasible or extremely costly. Moreover, the resulting sampling algorithm does not require density estimation, and it is entirely particle-based. 

The first main contribution of this work is to establish exponential convergence to the target Gibbs distribution for a broad class of dynamics in the form of~\eqref{EQ:AdVar FPE} with adaptive diffusion; see Theorem~\ref{thm:general_rate}. We show that these dynamics can be regarded as weighted Wasserstein gradient flows of the same energy functional under different weighted Wasserstein metrics. The convergence of the class of dynamics to the same target Gibbs distribution~\eqref{EQ:Gibbs} is quantified in terms of both the Kullback--Leibler (KL) divergence and the $\chi^2$ divergence. An interesting consequence of our main result is a new convergence rate for the overdamped Langevin dynamics (with a constant diffusion coefficient). While the classical results depend on the curvature of the Gibbs potential (see Theorem~\ref{thm:langevin}), our theorem instead yields a rate that depends on the variation of the Gibbs potential.

The second key contribution of this work is to demonstrate that not only can Gibbs distribution sampling be achieved in a derivative-free manner, but it also leads to a significantly faster rate of convergence compared to overdamped Langevin dynamics when sampling non-log-concave distributions. We demonstrate this advantage in two ways. First, based on Theorem~\ref{thm:general_rate}, the upper bound for the convergence rate to the target Gibbs distribution reveals that the derivative-free dynamics are more effective when the Gibbs potential is non-convex (see Section~\ref{subsec:rates_compare}). Second, in Section~\ref{sec:time}, we use a one-dimensional case study to compute the mean exit time for these two dynamics when transitioning from a local basin of attraction to another one. The derivative-free dynamics require $\mathcal{O}\left(\frac{1}{\epsilon}\right)$ time, whereas the overdamped Langevin dynamics take $\mathcal{O}\left(\exp\left(\frac{1}{\epsilon}\right)\right)$ time, which is significantly longer for very small $\eps$ (that is, when we sample distributions that are highly concentrated).

The paper is organized as follows: In Section~\ref{sec:WGF}, we introduce a new perspective on sampling algorithms in the class~\eqref{EQ:AdVar Cont} by treating them as different types of \textit{weighted} Wasserstein gradient flows, where the energy functional of $\rho$ is the KL divergence between $\rho$ and the target Gibbs distribution. This perspective allows us to extend the classic overdamped Langevin dynamics to various linear gradient flows with state-dependent diffusion coefficients. Section~\ref{sec:conv} characterizes the convergence behavior of these dynamics in terms of $\chi^2$ divergence and KL divergence, establishing non-asymptotic exponential rates and comparing their performance. In Section~\ref{sec:time}, we analyze two specific examples---the overdamped Langevin dynamics and the derivative-free dynamics---from a different perspective, comparing their (hill-climbing) mean exit times from an interval starting from the same point in a local valley. For the overdamped Langevin dynamics, the mean exit time depends on the energy gap between the local minimum and maximum, whereas interestingly for the derivative-free dynamics, it does not. Section~\ref{sec:numerics} provides examples demonstrating the sampling convergence rates and mean exit times, which align with the theory presented in the earlier sections. Finally, conclusions are presented in Section~\ref{sec:conclusion}.

\section{Adaptive Diffusion as Weighted Wasserstein Gradient Flows}\label{sec:WGF}

An important observation is that the Fokker--Planck equation~\eqref{EQ:AdVar FPE} can be rewritten in terms of the divergence form
\begin{equation*}
     \partial_t \rho =  \nabla \cdot \left(\rho 
 \left( \nabla H+ D \nabla \log \left( D \,\rho \right) \right) \right) .
\end{equation*}
When the diffusion coefficient $D(\bx) \equiv \eps$, the drift term $\nabla H(\bx)$ has to be $\nabla  F(\bx)$ to keep the Gibbs measure as a steady state. The corresponding dynamics becomes the classic overdamped Langevin diffusion. It is well-known that, in the case of $(H(\bx),D(\bx))=(F(\bx),\eps)$, \eqref{EQ:AdVar FPE} can be understood as a Wasserstein gradient flow equation thanks to the seminal ``JKO'' paper~\cite{jordan1998variational}:
\begin{equation}\label{eq:overdamped Langevin}
    \partial_t \rho=  \nabla (\rho \nabla F ) + \eps \Delta \rho = -\nabla_{W_2}\cE(\rho) =  \nabla \cdot \left( \rho \nabla \frac{\delta \cE}{\delta \rho}\right)\,,
\end{equation}
where the energy functional $\cE(\rho)$ is proportional to the Kullback--Leibler (KL) divergence between $\rho$ and $\pi_G$, and $\pi_G$ is the Gibbs distribution~\eqref{EQ:Gibbs}. That is,
\begin{equation}\label{eq:energy}
    \cE(\rho)= \eps \int \rho (\bx) \log\rho(\bx)  d\bx  +\int F(\bx)\rho(\bx) d\bx  = \eps\,  \text{KL}\left( \rho || \pi_G \right)\,.
\end{equation}

If $F(\bx)$ is a strongly convex function with respect to $\bx$, $\cE(\rho)$ is also strictly  convex along the Wasserstein geodesics, also known as displacement  or geodesically convex~\cite{ambrosio2005gradient}.  Consequently, one can establish various exponential convergence results in the form of~\eqref{eq: expo converge} with $d$ being the $\chi^2$ divergence, the  KL divergence, and the quadratic Wasserstein metric ($W_2$). However, it is worth addressing that many of the results may fail to hold if $F(\bx)$ is not strongly convex, which is equivalent to the steady-state distribution $\pi_G$ being non-log-concave.

Here, we generalize the JKO formulation and recast an entire class of dynamics in the form of~\eqref{EQ:AdVar FPE} as a weighted Wasserstein gradient flow of the energy functional $\cE(\rho)$ given in~\eqref{eq:energy}. To begin with, we define a general type of weighted quadratic Wasserstein metric $\mathcal{W}_w$  with a weight function $w(x,\rho)$:
\begin{equation}\label{eq:modified_W2_2}
    \mathcal{W}_w (\mu_0, \mu_1) = \sqrt{\inf_{(\rho_t, {\bf m}_t)} \left\{  \int_0^1 \int  \frac{| {\bf{m}}_t|^2}{\rho }\frac{\rho}{w(x,\rho)} d x d t  \right\}}\,,
\end{equation}
where $ {\bf m}_t$ represents the momentum and the pair $(\rho_t,   {\bf m}_t)$ is subject to the classic continuity equation
\begin{equation}\label{eq:continuity}
    \partial_t \rho_t + \nabla \cdot  {\bf{m}}_t  = 0,\quad d\mu_0 = \rho(x,t=0) dx ,\quad d\mu_1 = \rho(x,t=1) dx. 
\end{equation}
Note that if $w(x,\rho) = \rho =: w_1$, i.e., $\int \frac{\rho}{w(x,\rho)} d x = |\bbT^d|$, the distance given in~\eqref{eq:modified_W2_2} reduces to the conventional quadratic Wasserstein metric~\cite{Villani-Book03}.

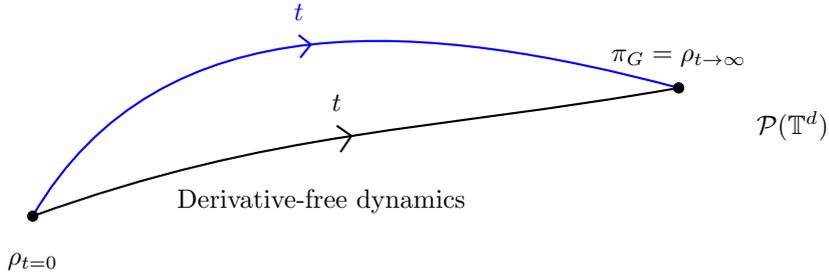
\begin{figure}
    \centering
\begin{tikzpicture}

\draw[thick,blue] (1.5, 1.3) to [out=60, in=165] coordinate[pos=0.5](B) (10, 3);
\draw[thick,blue,-] (B) -- (5.02,3.4) ($(10, 3)-(1.5, 1.3)$);
\draw[thick,blue,-] (B) -- (5,3.7) ($(10, 3)-(1.5, 1.3)$);

\draw[thick,black] (1.5, 1.3) to [out=20, in=190] coordinate[pos=0.5](C) (10, 3);
\draw[thick,black,-] (C) -- (5.55,2.15) ($(10, 3)-(1.5, 1.3)$);
\draw[thick,black,-] (C) -- (5.5,2.5) ($(10, 3)-(1.5, 1.3)$);

\node [circle,fill,inner sep=1.5pt]  at (1.5, 1.3) {};
\node [circle,fill,inner sep=1.5pt]  at (10, 3) {}; 

\node[blue] at (5, 4) {$t$};
\node[] at (5.5, 2.8) {$t$};

\node[] at (1.5, 0.7) {$\rho_{t=0}$};

\node[] at (10, 3.4) {$\pi_G = \rho_{t\rightarrow\infty}$};

\node[] at (11.5, 2.5) {$\mathcal{P}(\bbT^d)$};

\node[text=blue] at (5, 4.7) {Overdamped Langevin dynamics};

\node[text=black] at (5.3, 1.5) {Derivative-free dynamics};
\end{tikzpicture}
\caption{Overdamped Langevin~\eqref{eq:overdamped Langevin} and the derivative-free~\eqref{eq:diffuse_div_free} dynamics are two weighted Wasserstein gradient flows of the energy $\cE(\rho) = \text{KL}(\rho||\pi_G)$, yielding two curves in the space of probability distributions $\mathcal{P}(\bbT^d)$ starting from the same initial distribution $\rho_0$. Their convergence properties depend on different features of the Gibbs distribution $\pi_G$~\eqref{EQ:Gibbs}.}
    \label{fig:diagram}
\end{figure}

This type of weighted Wasserstein metric has been studied in the literature, in particular, to analyze nonlinear evolution equations~\cite{dolbeault2009new,carrillo2010nonlinear,lisini2009nonlinear,lisini2012cahn}. For any energy functional $U(\rho)$, the gradient flow of $U$ with respect to the geometry induced by $\mathcal{W}_w$ is 
\begin{equation}\label{eq:weighted W2 GF}
\partial_t \rho = -\nabla_{\mathcal{W}_{w}} U(\rho) =  \nabla \cdot \left( w(x,\rho) \nabla\frac{\delta U}{\delta \rho}   \right)\,.
\end{equation}
Many weighed/modified Wasserstein gradient flow have been studied, and we mention a few as examples. If $w(x,\rho) = \rho^\alpha$ and $U(\rho) = \int V \rho$ for some smooth potential $V$, we can view the scalar conservation law $\partial_t\rho = \nabla \left(\rho^\alpha \nabla V\right)$ as a weighted Wasserstein flow~\cite{dolbeault2009new}. In~\cite{lisini2012cahn}, the Cahn--Hilliard and certain thin-film equations are studied as gradient flow  equations for a perturbed Dirichlet energy with respect to certain weighted Wasserstein transport metrics. In particular, nonlinear diffusion equations with variable coefficients are studied in~\cite{lisini2009nonlinear} as gradient flows induced by the $2$-Wasserstein metric endowed with
the Riemannian distance. This perspective is very relevant to our work.

In this paper, we are interested in sampling the Gibbs distribution~\eqref{EQ:Gibbs}, which is the global minimizer of the energy functional $\cE(\rho)$ defined in~\eqref{eq:energy}. Using the first-order variation, we can compute 
\begin{equation*}
   \nabla  \frac{\delta \cE}{\delta \rho} = \nabla F(x) + \eps\, \nabla  \log \rho  \,,
\end{equation*}
and the general weighted Wasserstein gradient flow equation~\eqref{eq:weighted W2 GF} takes the form
\begin{equation}~\label{eq:weighted W2 GF KL}
    \partial_t \rho = - \eps \nabla \cdot \left( \rho \nabla   \left( \frac{w}{\rho} \right) \right) + \eps \Delta w + \nabla \cdot (w \nabla F)\,.
\end{equation}
In particular, if we assume that $w(\bx,\rho)$ is $1$-homogeneous with respect to $\rho$, i.e., 
\begin{equation}\label{eq:one-homo}
w(\bx,\rho) = f(\bx) \rho(\bx)
\end{equation}
for some positive function $f(\bx)$, the gradient flow~\eqref{eq:weighted W2 GF KL} is further reduced to
\begin{equation}~\label{eq:weighted W2 GF KL_1}
    \partial_t \rho = - \nabla \cdot\left( \rho \left( \eps \nabla f - f \nabla F \right) \right) + \eps \Delta \left(f \rho \right)\,.
\end{equation}
Additionally, if for some function $g: [F_{\min}, F_{\max}] \to \bbR_+$ with $F_{\min}$, $F_{\max}$ being the minimum and maximum value of $F(\bx)$ on the domain $\bbT^d$,  we have that 
\begin{equation}\label{eq:important_relation}
f = g\circ F\,,\quad D= \eps f \,,\quad  \text{and}\quad H(\bx) =  -D(\bx) + \int_{F_{\min}}^{F(\bx)} g(y)dy \,,
\end{equation}
then the weighted Wasserstein gradient flow equation~\eqref{eq:weighted W2 GF KL_1} becomes
\begin{equation*}
     \partial_t \rho =  \nabla \cdot \left( \rho \nabla H \right)   +   \Delta \left( D \rho \right)  \,,
\end{equation*}
which is the Fokker--Plank equation given in~\eqref{EQ:AdVar FPE}.

So far,  we have reinterpreted this class of adaptive diffusion dynamics as weighted Wasserstein gradient flows. Here,  we mention a few examples:
\begin{itemize}
    \item If $g\equiv 1$, then $D = \eps$ and $\nabla H =  \nabla F$. We recover the overdamped Langevin dynamics $\partial_t\rho =  \nabla \cdot \left(\rho \nabla F \right) + \eps \Delta \rho$.
    \item If $g = |\bbT^d|^{-1} Z_G\exp\left(\frac{x}{\eps} \right)$, then $D(\bx) = \eps |\bbT^d|^{-1} Z_Ge^{\frac{F(\bx)}{\eps}} $ and $\nabla H \equiv 0$. We recover the derivative-free dynamics  $\partial_t\rho =  \eps |\bbT^d|^{-1} Z_G \Delta \,  \left( e^{\frac{F(\bx)}{\eps}}   \rho\right) $.
    \item If $g(x) = x$, we have $D(\bx) = \eps F(\bx)$ and $\nabla H(\bx) = \left(   F(\bx) - \eps \right) \nabla  F(\bx) $, we obtain a PDE with a nonzero drift and a state-dependent diffusion coefficient:
    \[
    \partial_t\rho = \nabla \cdot \left( \left( F(\bx)-\eps \right) \nabla F(\bx) \rho\right)  + \eps \Delta (F(\bx) \rho )\,.
    \]
\end{itemize}

We want to point out that, for any proper choice of the weight function $w$, the Gibbs distribution $\pi_G \varpropto \exp\left(-\frac{F(\bx)}{\eps}\right)$ is always one steady state  of the weighted Wasserstein gradient flow~\eqref{eq:weighted W2 GF KL}.
This is due to the fact that $\frac{\delta \cE}{\delta \rho}  = 0$ when $\rho = \pi_G$. If we want to have a \textbf{linear diffusion term}, we can choose weights $w(\bx,\rho)$ that are $1$-homogeneous in $\rho$ in the form of~\eqref{eq:one-homo}. Then for any proper choice of variable-dependent function $f(\bx)$, $\pi_G$ will always be one steady state solution to~\eqref{eq:weighted W2 GF KL_1}. 
To avoid the artificial re-scaling of time, we further restrict that
\[
\int_{\bbT^d} f(\bx)^{-1} d\bx =  \int_{\bbT^d} 1 d\bx = |\bbT^d|\,. 
\]

A more intuitive interpretation of the weighted $W_2$ metric under the $1$-homogeneous weight~\eqref{eq:one-homo} is to view it as the Wasserstein metric induced by the Riemannian distance on $\bbT^d$ defined by
\[
d(x,y) = \left\{ \int_0^1\sqrt{\langle  \, f(\gamma(t))  \,\dot{\gamma}(t), \dot{\gamma}(t) } dt:\,\, \gamma \in \text{AC}([0,1];\bbT^d),\,\, \gamma(0) = x\,, \gamma(1) = y\right\}\,, 
\]
where $\text{AC}([0,1];\bbT^d)$ denotes the set of absolutely continuous curves in $\bbT^d$ parameterized in the time interval $[0,1]$. Then we can equivalently write down
\[
\mathcal{W}_{w}  (\mu_0,\mu_1) = \sqrt{\inf_{\pi \in \Pi(\mu_0,\mu_1)}\int d^2(x,y) d\pi(x,y)}\,,
\]
where $\Pi(\mu_0,\mu_1)$ is the set of all the coupling between measures $\mu_0$ and $\mu_1$. When $f(\bx) \equiv 1$, the Riemannian distance $d(x,y)$ reduces the standard Euclidean distance.

In the derivative-free case, \eqref{eq:weighted W2 GF KL_1} takes the following form:
\begin{align}
     \partial_t \rho&  = C  \nabla \cdot \left(\rho D \nabla \left( \log \left(  D \,\rho \right)\right) \right) \nonumber \\
     &= C \Delta (D \rho )\,,\quad C = |\bbT^d|^{-1} Z_G \,\eps\,,\label{eq:diffuse_div_free}
\end{align}
 where $Z_G$ is the normalizing constant 
 of the Gibbs measure given in~\eqref{EQ:Gibbs}. We remark that the constant scaling $|\bbT^d|$ in $C$ is necessary to ensure consistency of the time scale.

\section{Exponential Convergence Comparison}\label{sec:conv}

Despite that the entire class of dynamics~\eqref{EQ:AdVar FPE} all has the potential to converge to the target distribution~\eqref{EQ:Gibbs} since the Gibbs distribution is one steady state, we are interested in those that may exhibit exponential convergence along the trajectory.  We will provide the convergence rate for the whole class of dynamics in this section,  but  also have extra discussions on two special evolution PDEs:
\begin{eqnarray*}
    \text{Overdamped Langevin dynamics~\eqref{eq:overdamped Langevin}:}\qquad \partial_t \rho &=& -\nabla_{\mathcal{W}_{w_1}} \cE(\rho), \quad w_1(x,\rho) = \rho \,. \\ %
    \text{Derivative-free dynamics~\eqref{eq:diffuse_div_free}:}\qquad \partial_t \rho &=&   -\nabla_{\mathcal{W}_{w_2}} \cE(\rho),\quad w_2(x,\rho) = |\bbT^d|^{-1}\rho/\pi_G \,. %
\end{eqnarray*}
See Figure~\ref{fig:diagram} for an illustration.

\subsection{Convergence of adaptive diffusion}

In this subsection, we examine the convergence behavior of the class of weighted Wasserstein gradient flow equation~\eqref{eq:weighted W2 GF KL_1} in terms of the KL divergence and the $\chi^2$ divergence between the current time solution to the steady-state solution $\pi_G$ given in~\eqref{EQ:Gibbs}. The results are the counterparts of~\Cref{thm:langevin}. This analysis covers the derivative-free dynamics~\eqref{eq:diffuse_div_free}.

\begin{theorem}\label{thm:general_rate}
    Consider Equation~\eqref{eq:weighted W2 GF KL_1}, the weighted Wasserstein gradient flow of energy $\cE(\rho) = \text{KL}(\rho|\pi_G)$, i.e.,
    \begin{equation*}
        \partial_t \rho  = \nabla \cdot \left(f\,\rho\, \left(\nabla F + \eps \nabla \log \rho\right)\right)\,.
    \end{equation*}
    Assume that the initial distribution has enough regularity such that the strong solution to the equation exists for $t\in [0, \infty)$. Then the dynamics converges exponentially fast to the unique steady-state distribution $\pi_G$ given in~\eqref{EQ:Gibbs}:
    \begin{eqnarray}
        \chi^2(\rho(\bx,t)||  \pi_G) &\leq & \exp\left( -2 \lambda_1   t\right) \, \chi^2(\rho(\bx,0)||  \pi_G) \,,\label{eq:free PI}\\
        \text{KL}(\rho(\bx,t)||  \pi_G) &\leq &  \exp\left(- 
 2 \lambda_2 t \right) \Big(  2 \lambda_2 t  \log \left( D_{\max}Z_G\right)  + \text{KL}(\rho(\bx, 0) || \pi_G) \Big) \label{eq:free LSI}
        \end{eqnarray}
where 
\begin{equation}\label{eqn:lambdas}
\lambda_1 =   \frac{\lambda_\eps}{C_{\text{PI}}}  \ , 
\quad \lambda_2 =\frac{\lambda_\eps}{C_{\text{LSI}}}   \quad \text{with}\quad \lambda_\eps= \frac{\eps  D_{\min}^2}{D_{\max}}\min_{\bx\in\bbT^d} f(\bx)e^{-\frac{F(\bx)}{\eps}}\,.
\end{equation}
Here, $C_{\text{PI}}$ and $C_{\text{LSI}}$  are the Poincar\'e inequality and the Log-Sobolev inequality constants for the Lebesgue measure over $\bbT^d$, respectively, and $D_{\min} = \min_{\bx\in\bbT^d} e^{\frac{F(\bx)}{\eps}}$, $D_{\max} = \max_{\bx\in\bbT^d} e^{\frac{F(\bx)}{\eps}}$. 
\end{theorem}
\begin{proof}
We first prove the case of $\chi^2$ divergence given in~\eqref{eq:free PI}. Let $v=\rho-\pi_G$. Then $v$ solves
\begin{equation*}
    \partial_t v= \nabla \cdot \left(f\,v\, \left(\nabla F + \eps \nabla \log v\right)\right)\,,
\end{equation*}
since~\eqref{eq:weighted W2 GF KL_1} is linear with respect to $\rho$,  and $\pi_G$ is a steady-state solution. It is worth noting that $\int v(\bx,t) d\bx \equiv 0 $ for all $t$. We comment that $v(\bx,t)$, $\rho(\bx,t)$ for any fixed $t$, together with $\pi_G(\bx)$ and $F(\bx)$, are  functions over the torus $\bbT^d$. We further have 
\begin{eqnarray*}
    \frac{d}{dt}\left(\int_{\bbT^d} e^{\frac{F(\bx)}{\eps}} v^2(\bx, t)d\bx \right)&=&-2 \eps \int_{\bbT^d} fv \, \nabla \left( e^{\frac{F}{\eps}} v \right)\cdot \nabla \log \left(e^{\frac{F}{\eps}}v\right) d\bx\\
    &=&-2 \eps \int_{\bbT^d} \Big|\nabla \left( e^{\frac{F}{\eps}} v\right) \Big|^2 f e^{-\frac{F}{\eps}} d\bx\\
    & \leq & -2 \eps D_{\min} \min_{\bx\in \bbT^d}\left(f e^{-\frac{F}{\eps}}\right) \, \int_{\bbT^d} \Big|\nabla \left( e^{\frac{F}{\eps}} v\right) \Big|^2 e^{-\frac{F}{\eps}} d\bx  \\
    &\leq & - \frac{2\eps}{C_{\text{PI}}} \frac{ 
D_{\min}^2 }{D_{\max}  }\min_{\bx\in \bbT^d}\left(f e^{-\frac{F}{\eps}} \right)  \int_{\bbT^d} e^{\frac{F(\bx)}{\eps}} v^2(\bx, t)d\bx\,,
\end{eqnarray*}
where $D_{\min}$ and $D_{\max}$ are the minimum and maximum values of $e^{\frac{F}{\eps}}$ on the domain $\bbT^d$. In the last step above, we have applied the weighted Poincar\'e inequality in Theorem~\ref{thm:wpi}. In our case, we consider the weight function
$$
    w(\bx) = f(\bx)e^{-\frac{F(\bx)}{\eps}}
$$ 
in the Muckenhoupt $A_2$ class (see~\Cref{DEF:Ap}). We can find an upper bound of its $A_2$ constant $[w]_2$ as
\begin{equation}
\label{eq:constant_our_setting}
    [w]_2 = \sup_{Q\subset \bbT^d} \left({\frac {1}{V_Q }}\int _{Q}w (\bx)\,d\bx\right)\left({\frac {1}{V_Q}}\int _{Q}w (\bx)^{-1}\,d\bx\right) \leq \frac{\max_{\bx\in \bbT^d} w(\bx) }{\min_{\bx\in \bbT^d} w(\bx)  }  
    = \frac{D_{\max}}{D_{\min}}\,.
\end{equation}
Based on~\eqref{eq:wpi} in the Appendix,  for the hypercube $\bbT^d$ and a Lipschitz function $e^{\frac{F}{\eps}}v$ satisfying $\int_{\bbT^d}(e^{\frac{F}{\eps}}v) \, e^{-\frac{F}{\eps}}d\bx = \int_{\bbT^d}v(\bx,t)  d\bx = 0$, we have
\[
\int_{\bbT^d} |e^{\frac{F}{\eps}}v|^2 e^{-\frac{F}{\eps}}  d\bx  \leq  C_{\text{PI}} \frac{D_{\max}}{D_{\min}}  \int_{\bbT^d} \Big|\nabla \left(e^{\frac{F}{\eps}}v \right)\Big|^2 e^{-\frac{F}{\eps}} d\bx\,,
\]
where $C_{\text{PI}}$ is the Poincar\'e inequality constant with respect to the Lebesgue measure, only depending on the domain $\bbT^d$:
\[
\int_{\bbT^d} |v|^2  d\bx  \leq  C_{\text{PI}} \int_{\bbT^d} |\nabla v|^2 d\bx\,,\quad \forall v \,\,\,\text{s.t.~} \int v d\bx  = 0\,.
\]
Note that 
\[
\int_{\bbT^d} e^{\frac{F(\bx)}{\eps}} v^2(\bx, t)d\bx = \|\rho - \pi_G\|_{L^2(e^{F/\eps})}^2 \varpropto \chi^2(\rho \| \pi_G ) := \left\|\frac{\rho}{\pi_G} - 1\right\|_{L^2(\pi_G)}^2\,,
\]
where $\chi^2(\rho \| \pi_G )$ is the chi-squared divergence between $\rho$ and $\pi_G$.
By Gr\"onwall's inequality, we have~\eqref{eq:free PI}.

To get the convergence result~\eqref{eq:free LSI} regarding the KL-divergence, we differentiate the energy $\cE(\rho)$ given in~\eqref{eq:energy} with respect to $t$. This leads to
\begin{eqnarray*}
\frac{d}{dt} \text{KL}(\rho || \pi_G)   &=&  \int   \left( 1 + \log \rho - \log \pi_G \right)  \, \partial_t \rho \, d\bx\\
&=& \eps \int  \left( 1 + \log \rho - \log \pi_G \right) \nabla \cdot \left( f \rho \nabla \left( \log \left( e^{\frac{F}{\eps}}\,\rho \right)\right) \right)\, d\bx\\
&=&  -   \eps \int f\,\rho  |\nabla \left( \log \left( e^{\frac{F}{\eps}}\,\rho \right)\right) |^2 d\bx \\
&\leq &  -  \eps D_{\min} \min_{\bx\in\bbT^d}\left(f e^{-\frac{F}{\eps}} \right) \int v  |\nabla  \log  v  |^2 d\bx\,,\quad v = c^{-1 } e^{\frac{F}{\eps}}\,\rho ,\quad c= \int_{\bbT^d} e^{\frac{F}{\eps}}\,\rho \,  d\bx \\
&\leq &  - \frac{2 \eps D_{\min} \min_{\bx\in\bbT^d}\left(f e^{-\frac{F}{\eps}}\right) }{C_{\text{LSI}}}  \int v \log v \,d\bx \,.
\end{eqnarray*}
In the last step above, we used the log-Sobolev inequality for the Lebesgue measure over $\bbT^d$: 
\[
\int_{\bbT^d} v \log v \, d\bx \leq  \frac{C_{\text{LSI}}}{2} \int_{\bbT^d} v  |\nabla  \log  v  |^2 d\bx\,,\quad  \text{for\,}\int_{\bbT^d} v\, d\bx = 1\,.
\]
The key constant $C_{\text{LSI}}$ only depends on the diameter of $\bbT^d$~\cite{mikulincer2024brownian}. We address that we are not using the log-Sobolev inequality with respect to the Gibbs measure $\pi_G$, which could potentially yield a much worse constant.

Note that 
\begin{eqnarray*}
    \int v\log v d\bx & = & \int c^{-1} e^{\frac{F}{\eps}}\rho \left(\log (e^{\frac{F}{\eps}}\rho) - \log c\right)  d\bx \\
    &\geq & \frac{D_{\min}}{D_{\max}}\left(\int  \rho \log (e^{\frac{F}{\eps}} \rho)  d\bx  - \log  c \right) \\
    &= & \frac{D_{\min}}{D_{\max}}\left(\int  \rho \log (e^{\frac{F}{\eps}} \rho)  d\bx  - \log  c \right) + \frac{D_{\min}}{D_{\max}} \int \rho \log Z_G d\bx  - \frac{D_{\min}}{D_{\max}} \log Z_G\\
    &=& \frac{D_{\min}}{D_{\max}} \text{KL}(\rho || \pi_G)  - \frac{D_{\min}}{D_{\max}}  \log \left( c\, Z_G \right)\,.
\end{eqnarray*}
Combining the above inequalities together, we  have
\[
\frac{d}{dt} \text{KL}(\rho || \pi_G)  \leq -\ 2\lambda_2 \text{KL}(\rho || \pi_G)  +   2\lambda_2  \log \left( c Z_G \right)\,,
\]
\[
\text{KL}(\rho(\cdot,t)) || \pi_G)  \leq  \text{KL}(\rho(\cdot,0)) || \pi_G) + 2\lambda_2 \int_0^t   \log \left( c(s) Z_G \right) ds + \int_0^t -  2\lambda_2  \text{KL}(\rho(\cdot, s) || \pi_G)  ds \,,
\]
where $\lambda_2$ is given in~\eqref{eqn:lambdas}.
We remark that $ c Z_G = \left(\int e^{\frac{F}{\eps}} \rho d\bx \right) \left( \int e^{-\frac{F}{\eps}}d\bx\right)  \geq 1 $ by H\"older Inequality. Thus,  $\log \left( c Z_G \right) \geq 0$. By Gr\"onwall's inequality, we have that 
\begin{eqnarray*}
     \text{KL}(\rho(\bx, t) || \pi_G) &\leq& \left(  2\lambda_2  \int_0^t \log \left( c(s) Z_G \right)  ds + \text{KL}(\rho(\bx, 0) || \pi_G) \right)  \exp\left(- 2\lambda_2 t \right)\, \nonumber \\
     &\leq & \Big(    2\lambda_2 t \log \left( D_{\max}Z_G\right) + \text{KL}(\rho(\bx, 0) || \pi_G) \Big) \exp\left(- 2\lambda_2 t \right) \, ,
\end{eqnarray*}
showing exponential convergence in the KL-divergence~\eqref{eq:free LSI}.
\end{proof}

\subsection{Convergence of the overdamped Langevin dynamics}

The convergence of the overdamped Langevin dynamics~\eqref{eq:overdamped Langevin} can be established in terms of both the KL divergence and the $\chi^2$ divergence. The former requires the Gibbs distribution $\pi_G$ to satisfy the Poincar\'e inequality, while the latter requires that the log-Sobolev inequality to hold. 

Consider $\pi_G$ with a compact support $\Omega :=\text{supp}(\pi_G) $ on which
$\pi_G \varpropto \exp\left(-\frac{F(\bx)}{\eps}\right)$. If on the interior of $\Omega$, the Hessian of $F(\bx)$ satisfies
\begin{equation}\label{eq:kappa-log-concave}
D^2 F(\bx) \succeq \kappa \,\text{Id},\quad \kappa \in \bbR,
\end{equation}
then we say $\pi_G$ is $\frac{\kappa}{\eps}$-log-concave. Under this condition, both the Poincar\'e inequality and the log-Sobolev inequality hold for $\pi_G$~\cite{mikulincer2024brownian}.

The following theorem regarding the overdampled Langevin dynamics is a simple exercise of variational analysis. We provide the proof sketch here for completeness.

\begin{theorem}\label{thm:langevin}
    Assume that  the potential function $F(\bx)$ of the overdamped Langevin diffusion~\eqref{eq:overdamped Langevin} satisfies~\eqref{eq:kappa-log-concave} with the constant $\kappa\in \mathbb{R}$ on the domain $\bbT^d$. Assume that the initial distribution has enough regularity such that the strong solution to the equation exists for $t\in [0, \infty)$.
    Then the time-dependent PDE solution  converges to the Gibbs distribution $\pi_G$ given in~\eqref{EQ:Gibbs}  exponentially fast:
    \begin{eqnarray}
        \chi^2(\rho(\bx,t)||\pi_G) &\leq & \exp\left( -2\lambda  t\right) \, \chi^2(\rho(\bx,0)||\pi_G) \,,\label{eq:langevin PI}\\
        \text{KL}(\rho(\bx,t)||\pi_G) &\leq & \exp\left( - 2\lambda  t\right) \, \text{KL}(\rho(\bx,0)||\pi_G)\,,\label{eq:langevin LSI}
        \end{eqnarray}
where the rate
\begin{equation}\label{eq:langevin_constants}
    \lambda= \begin{cases}
        \kappa  & \text{if}\,\,\, \kappa \, L^2 \geq \eps \,,\\
       \dfrac{2\eps}{L^2 + L^2e^{1 - \tfrac{\kappa L^2}{\eps}} } & \text{if}\,\,\, \kappa \, L^2 < \eps \,,
    \end{cases}
\end{equation}
and $L = \text{diam}(\bbT^d)$.
\end{theorem}
\begin{proof}
We sketch the proof below. Define $v = \rho - \pi_G$. Note that 
\[
\int_{\bbT^d} v^2(\bx, t)  \pi_G^{-1} d\bx = \|\rho - \pi_G\|_{L^2( \pi_G^{-1})}^2 \varpropto \chi^2(\rho \| \pi_G ) = \left\|\frac{\rho}{\pi_G} - 1\right\|_{L^2(\pi_G)}^2\,,
\]
where $\chi^2$ is the chi-squared divergence. Let $\wt v =v / \pi_G$.
\begin{eqnarray*}
 \partial_t  \|\wt v\|_{L^2(\pi_G)}^2=   \partial_t  \int_{\bbT^d} v^2(\bx, t)  \pi_G^{-1} d\bx =  -2\eps\|\nabla \wt v\|_{L^2(\pi_G)}^2 \leq  -\frac{2\eps}{ C_{PI}(\pi_G)} \|\wt v\|_{L^2(\pi_G)}^2\,,
\end{eqnarray*}
where $C_{PI}(\pi_G)$ is the Poincar\'e inequality constant of $\pi_G$. Similarly,
\begin{eqnarray*}
 \partial_t  \text{KL}(\rho||\pi_G)=     
 - \eps  \int \nabla \log \left(\frac{\rho}{\pi_G}\right)^2 \rho \, d\bx \leq  -\frac{2 \eps }{C_{LSI}(\pi_G) }\text{KL}(\rho||\pi_G)\,,
\end{eqnarray*}
where $ C_{LSI}(\pi_G)$ is the Log-Sobolev inequality constant of $\pi_G$.
Based on~\cite[Table 1]{mikulincer2024brownian} for $C_{PI}(\pi_G)$ and $C_{LSI}(\pi_G)$, and Gr\"onwall's inequality, we obtain~\eqref{eq:langevin PI} and ~\eqref{eq:langevin LSI}.
\end{proof}

In most cases, the overdamped Langevin dynamics on a compact domain has exponential convergence, including nonconvex potential functions, unless the eigenvalue of $D^2 F(\bx)$ cannot be lower bounded. The convergence rate is also directly determined by the curvature of the potential function $F(\bx)$. 

When $\kappa<0$, i.e., $\pi_G$ is not log-concave, the convergence rate decreases to zero exponentially fast with respect to the diffusion constant $\eps$. This highlights the importance of annealing, i.e., adapting the diffusion constant $\eps$, in sampling multi-model distributions.

\subsection{Convergence comparisons}\label{subsec:rates_compare}

We have the following remarks about~\Cref{thm:general_rate} and~\Cref{thm:langevin} regarding the overdamped Langevin dynamics and the derivative-free dynamics~\eqref{eq:diffuse_div_free}.
\begin{enumerate}
    \item All results in \Cref{thm:general_rate} and~\Cref{thm:langevin} are upper bounds for the convergence, which might not be tight. However, 
    when the potential function $F(\bx)$ is constant, both theorems returns the convergence rate for the standard heat equation with $\lambda_1 = \frac{\eps}{C_{\text{PI}}}$ and  $\lambda_2 =  \frac{\eps}{C_{\text{LSI}}}$. Furthermore, the numerical examples in Section~\ref{sec:numerics} suggest that these upper bounds are close to being tight, as the observed exponential convergence rates agree with our theoretical predictions.
    \item Note that $\frac{Z_G}{|\bbT|^d}   \geq \frac{1}{D_{\max}}$, and thus 
$\lambda_\eps \geq \eps \exp\left(\frac{2}{\eps}\left(F_{\min} - F_{\max} \right)\right) $, 
where $F_{\min}$ and $F_{\max}$ are the minimum and maximum values of $F(\bx)$ over $\bbT^d$.   This reduces the convergence rate to a similar form to the one for the overdamped Langevin dynamics~\eqref{eq:langevin_constants} in the case of $\kappa < 0$: 
    \begin{eqnarray*}
    \text{Overdamped Langevin dynamics~\eqref{eq:overdamped Langevin}:}& \quad  &\mathcal{O}\left(\eps \exp\left(\frac{\kappa L^2}{\eps}\right)\right)\,, \label{eq:over-rmk}\\
    \text{Derivative-free diffusion~\eqref{eq:diffuse_div_free}:} &\quad&\mathcal{O}\left(\eps \exp\left(\frac{2(F_{\min}-F_{\max})}{\eps}\right)\right)\,.\label{eq:free-rmk}
    \end{eqnarray*}  
    If $F(\bx)$ is strictly concave, then we have $-\kappa L^2 \approx 2(F_{\min}-F_{\max})$. As a result, the two rates above are comparable. However, if $F(\bx)$ has many local minima and maxima, we could have 
$-\kappa L^2 \ll 2(F_{\min}-F_{\max})$.  Consequently, the derivative-free diffusion exhibits much better convergence performance with bigger rates, as seen in our numerical results (see Section~\ref{subset:conv}).
\item 
    One can view the overdamped Langevin dynamics discussed in Theorem~\ref{thm:langevin} as a special case of the general result in Theorem~\ref{thm:general_rate} with $f(\bx) \equiv 1$. The proof techniques are different in terms how the Log-Sobolev inequality and the Poincar\'e inequality are utilized. However, the result in Theorem~\ref{thm:general_rate} yields a rate of 
\begin{equation}\label{eq:Langevin=rate-general}
    \mathcal{O}\left(\eps \exp\left(\frac{2(F_{\min}-F_{\max})}{\eps}\right)\right)\,,
    \end{equation}
    independent of the curvature of $F(\bx)$.  Since both theorems provide upper bounds for the speed of $\rho(\bx, t)$ converging to the Gibbs measure $\pi_G$, they do not contradict each other.

\item In the limit of $\eps\rightarrow 0$, we can apply the Laplace's principle:
\begin{equation*}%
    Z_G = \int_{\bbT^d}\exp\left(- \frac{F(\bx)}{\eps}\right)d\bx  \varpropto \eps^{\frac{d}{2}}  \exp\left(- \frac{F_{\min}}{\eps}\right)  = \frac{\eps^{\frac{d}{2}}}{D_{\min}}\,.
    \end{equation*}
As a result, the rate for the derivative-free case is
\begin{equation}
    \mathcal{O}\left(\eps^{1+\frac{d}{2}} \exp\left(\frac{F_{\min}-F_{\max}}{\eps}\right)\right)\,.\label{eq:free-rmk-2}
\end{equation}
The exponential function's exponent in~\eqref{eq:free-rmk-2} is half that for~\eqref{eq:Langevin=rate-general}, implying a better rate of convergence for the derivative-free dynamics than the one for the overdamped Langevin dynamics.

Moreover, when $\eps\rightarrow 0$, the target Gibbs distribution converges to a Dirac delta centered at the global minimum of the Gibbs potential $F(\bx)$. The derivative-free sampling algorithm with fast convergence to the steady state can be utilized to perform global optimization. This builds the connection to our earlier work~\cite{engquist2022algebraically,engquist2024adaptive}.
\end{enumerate}

\section{Mean Exit Time }\label{sec:time}
In Section~\ref{sec:conv}, we compare the overdamped Langevin dynamics~\eqref{eq:overdamped Langevin} and the proposed derivative-free diffusion~\eqref{eq:diffuse_div_free} in terms of their convergence to the target Gibbs distribution $\pi_G$ given in~\eqref{EQ:Gibbs}. Although both dynamics exhibit exponential convergence in terms of both the KL divergence and the $\chi^2$ divergence, the interesting difference lies in the fact that the rate for the overdamped Langevin dynamics depends on the curvature of the potential $F(\bx)$ while the one for the derivative-free equation relies on energy variation of $F(\bx)$.

Results in Section~\ref{sec:conv} are in the form of upper bounds.  We would like to have a mean of comparison that is related to the ``lower bound'' for the amount of time it takes till the dynamics is close to the target measure.  Therefore,  in this section, we continue the comparison between these dynamics from a different angle: the mean exit time.  The mean exit time intuitively reflects how long it takes a particle to escape a local minimum of the nonconvex potential $F(\bx)$ so that the particle can visit the rest of the domain.

In particular,  we study the case of $\bx\in \bbR$ and the  domain of interest can be formed by the local minima of the potential function $F(\bx)$. This is often referred to as the hitting time, mean first-passage time or transition time in various setups of the same problem. Recall that the corresponding SDEs for the two dynamics~\eqref{eq:overdamped Langevin} and~\eqref{eq:diffuse_div_free} are
\begin{eqnarray*}
    \text{Overdamped Langevin dynamics~\eqref{eq:overdamped Langevin}:}\quad  
       dX_t &=& -\nabla F(X_t) dt + \sqrt{2\eps} \, dW_t\\
        \text{Derivative-free dynamics~\eqref{eq:diffuse_div_free}:}\quad 
    dX_t &=& \sqrt{2 \eps Z_G|\bbT^d|^{-1} \exp\left(F(X_t)/\eps\right)} \, dW_t
\end{eqnarray*}
It is often of interest to know how long a particle $X_t$ whose law follows the Fokker--Planck equation~\eqref{EQ:AdVar FPE} stays or exits a region of the domain. 

The following calculations are for the mean first-passage time problem in one dimension (1D) with the periodic boundary condition. Let $\mathcal T(x)$ be the time $X_t$ exits a domain $(a,b)$ with an initial position $X_0 = x \in (a,b)$. Define $G(x,t) $ as the probability of $\mathcal T(x) \geq t$. Then the mean first passage time is given by
\[
u(x) = \mathbb{E}[\mathcal{T}(x)] = \int_0^\infty G(x,t) dt\,.
\]
Moreover, $G(x,t)$ obeys the backward Kolmogorov equation, i.e., the adjoint of~\eqref{EQ:AdVar FPE}, with respect to the standard $L^2$ inner product~\cite[Sec.~5.2.7]{gardiner1985handbook}.

\subsection{Overdamped Langevin dynamics}
When $X_t$ follows the overdamped Langevin dynamics~\eqref{eq:overdamped Langevin}, the differential equation that $u(x)$ satisfies is:
\[ 
\eps u''(x) - F'(x) u'(x) + 1 = 0 \,.
\]
Here, we are interested in solving this with the following boundary condition 
\[
u(a) = u(b) = 0\,.
\]

Multiplying the equation for $u$ by the integrating factor \( e^{-F(x)/\eps} \), we have that
\[ 
e^{-F(x)/\eps} \left( \eps  u''(x) - F'(x) u'(x) + 1 \right) = 0\,.
\]
This leads to
\[ 
\frac{d}{dx} \left( \eps e^{-F(x)/\eps} u'(x) \right) + e^{-F(x)/\eps} = 0\,.
\]
Integrating both sides with respect to \( x \) gives us:
\[ 
    \eps e^{-F(x)/\eps} u'(x) + \int_a^x e^{-F(y)/\eps} dy = C_1\,,
\]
where \(C_1\) is an integration constant to be specified by the boundary condition.
Integrating again gives us
\[ 
u(x) = \frac{C_1}{\eps}\int_a^x \exp\left(\frac{F(y)}{\eps} \right) dy -\frac{1}{\eps}  \int_a^x \int_a^y \exp\left(\frac{F(y)-F(z)}{\eps} \right)dz dy \,.
\]
Since $u(b) = u(a) = 0$, we pin down the constant
\[ 
C_1=\Big(\int_a^b \exp\left(\frac{F(y)}{\eps} \right)dy\Big)^{-1}\int_a^b \int_a^y \exp\left(\frac{F(y)-F(z)}{\eps} \right) dz dy \,.
\]

\subsection{Derivative-free diffusion} 
For the derivative-free case, $u(x)$ satisfies
\[
\frac{\eps Z_G}{|\bbT^d|} \exp\left(\frac{F(x)}{\eps} \right)u''(x) = -1\,,\qquad u(a) = u(b) = 0\,.
\]
Simple calculation yields
\[
u(x) =  \frac{|\bbT^d|}{\eps Z_G}\frac{x- a}{b-a}  \int_a^b \int_a^y \exp\left(-\frac{F(z)}{\eps} \right)dz dy \, - \frac{|\bbT^d|}{\eps Z_G } \int_a^x \int_a^y \exp\left(-\frac{F(z)}{\eps} \right)dz dy \,.
\]

\subsection{Exit time comparison}

\begin{figure}
    \centering
    \subfloat[The potential $F(x)$]{\includegraphics[width=0.5\linewidth]{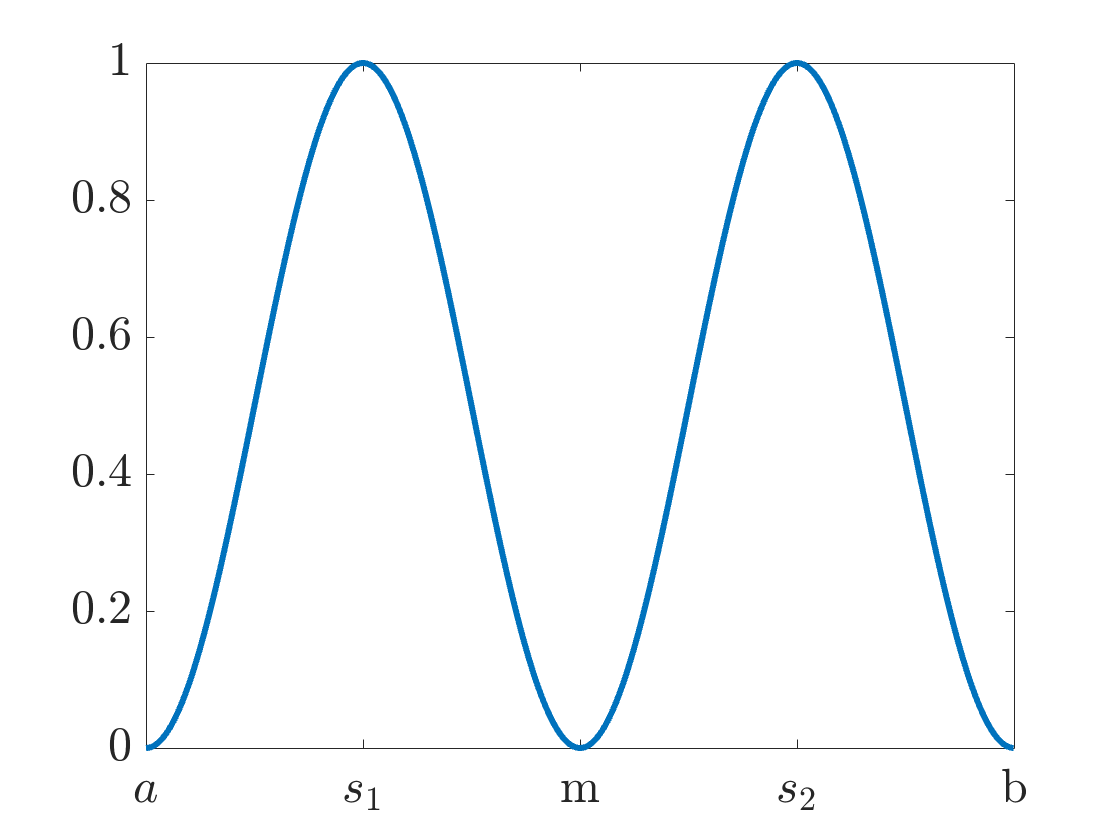}}
    \subfloat[The Gibbs distribution]{\includegraphics[width=0.5\linewidth]{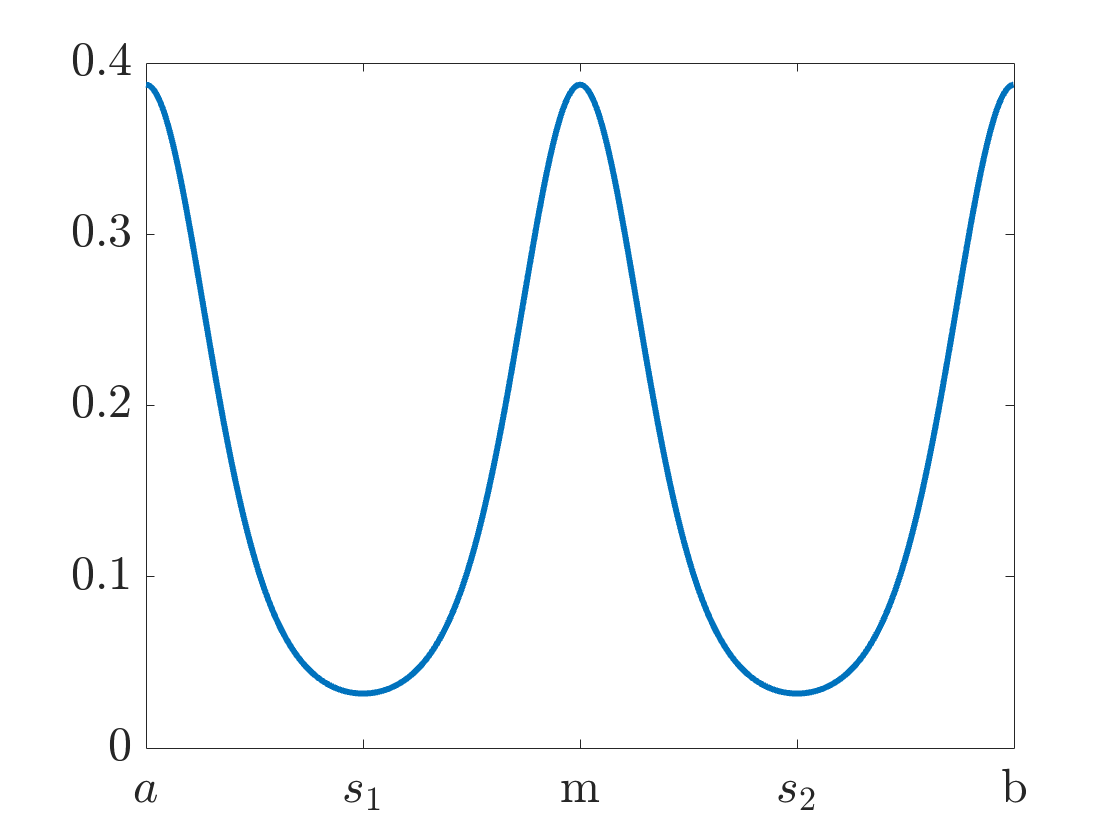}}
    \caption{A double-well potential (a) and its corresponding Gibbs distribution (b).}
    \label{fig:double-well}
\end{figure}

Let us consider the case where $F(x)$ is a double-well potential, with $m$ and $a$ (equivalently $b$) denoting the minima of the two wells, and $s_1$ and $s_2$ being the saddle points between the two wells. Without loss of generality, we assume that $F(s_1) = F(s_2)$, and $F''(s_1) = F''(s_2)$; see Figure~\ref{fig:double-well} for an illustration.  We are interested in how long it takes the dynamics to transit from $m$ to a neighborhood of $a$ or $b$. This correspond to setting $x = m$ in the mean-exit time function $u(x)$.

Assuming that near the saddling point $s_i$, $i = 1,2$, we have
\[
F(x) = F(s_i) - \frac{1}{2}\left(\frac{x-s_i}{\alpha_s}\right)^2\,,
\]
and near the interior local minimum  $m$, 
\[
F(x) = F(m) + \frac{1}{2}\left(\frac{x-m}{\alpha_m}\right)^2\,,
\]
then in the limit when $\eps \rightarrow 0$ we can approximate 
\begin{eqnarray*}
    \int_{s_1}^{s_2}  \exp \left(-\frac{F(z)}{\eps}\right)  dz &\approx & \alpha_m \sqrt{2\pi \eps}  \exp \left(-\frac{F(m)}{\eps}\right)\,, \\
     \int_{a}^{m}  \exp \left(\frac{F(z)}{\eps}\right)  dz &\approx& \alpha_s \sqrt{2\pi \eps}  \exp \left(\frac{F(s_1)}{\eps}\right) \approx \int_{m}^{b}  \exp \left(\frac{F(z)}{\eps}\right)  dz  \,.
\end{eqnarray*}
Under these assumptions, we can estimate on the mean exit time:
\begin{itemize}
    \item Langevin dynamics:
    \begin{eqnarray*}
        T_L(m \rightarrow a \text{ or } b) &\varpropto &  \alpha_s \alpha_m \exp \left(\frac{F(s_1) - F(m)}{\eps}\right) \,.
    \end{eqnarray*}
    \item Derivative-free dynamics:
    \begin{eqnarray}
        T_D(m \rightarrow a \text{ or } b) 
 &\varpropto&   \frac{ (b - m )(m - a) 
 \alpha_m }{\sqrt{\eps}Z_G} \exp \left(-\frac{F(m)}{\eps}\right) \,.   \nonumber
\end{eqnarray}
In the limit of $\eps \to 0$, 
$$
Z_G \approx 2 \alpha_a \sqrt{\pi \eps}\exp\left(-\frac{F(a)}{\eps}\right) + 2 \alpha_m \sqrt{\pi \eps}\exp\left(-\frac{F(m)}{\eps}\right)  ,
$$ implying
   \begin{equation*}
    T_D(m \rightarrow a \text{ or } b) \varpropto   \frac{ (b - m )(m - a)}{\eps} \exp \left(\frac{ 0\wedge   \left( F(a)- F(m)\right) }{\eps}\right)\,,
  \end{equation*}
if $\alpha_m \varpropto  \alpha_a$ where $1/\alpha_a$ is the second-order derivative of $F(x)$ at $a$.
\end{itemize}

Based on the calculations above,  if $F(a) = F(m)$, i.e., $x=a$ and $x=m$ are on the same level set of $F$, it takes $\mathcal{O}(\frac{1}{\epsilon})$ time for the derivative-free dynamics to escape one basin of attraction to another,  but it would take $\mathcal{O}(\exp\left(\frac{1}{\epsilon}\right))$ time for the overdamped Langevin dynamics!  Again,  from the mean exit time comparison, we see the significant advantage of the derivative-free dynamics in sampling non-log-concave distributions.

\section{Numerical Experiments}\label{sec:numerics}

In this section, we present several numerical experiments that illustrate the theoretical results established in Sections~\ref{sec:conv} and~\ref{sec:time}. Our focus is primarily on comparing the overdamped Langevin dynamics given by~\eqref{eq:overdamped Langevin} with the derivative-free dynamics described in~\eqref{eq:diffuse_div_free}. It is important to note that both of these dynamics are specific instances of the broader class of weighted Wasserstein gradient flows~\eqref{eq:weighted W2 GF KL_1}, where the weights are given by $f \equiv 1$ and $f \varpropto 1/\pi_G$, respectively.

\begin{figure} [!htbp]
    \centering
    \subfloat[The potential $F(x)$]{\includegraphics[width=0.5\linewidth]{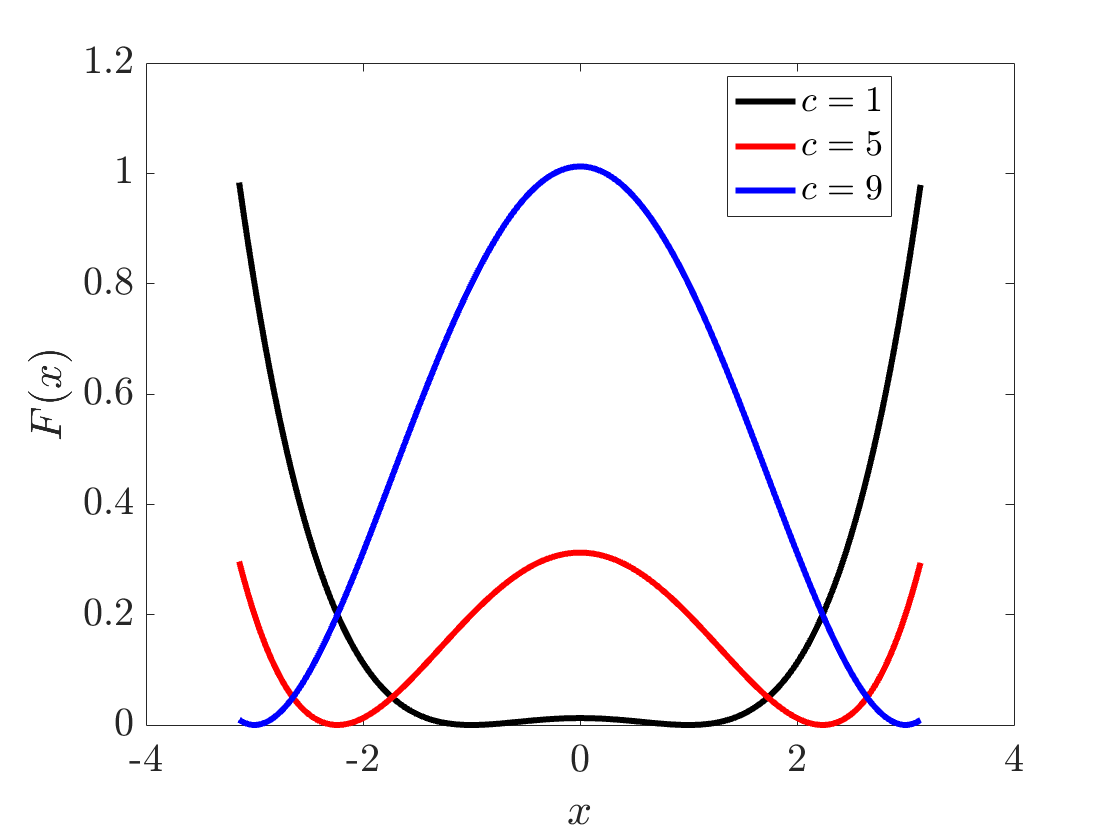}}
    \subfloat[The Gibbs distribution $\pi_G$]{\includegraphics[width=0.5\linewidth]{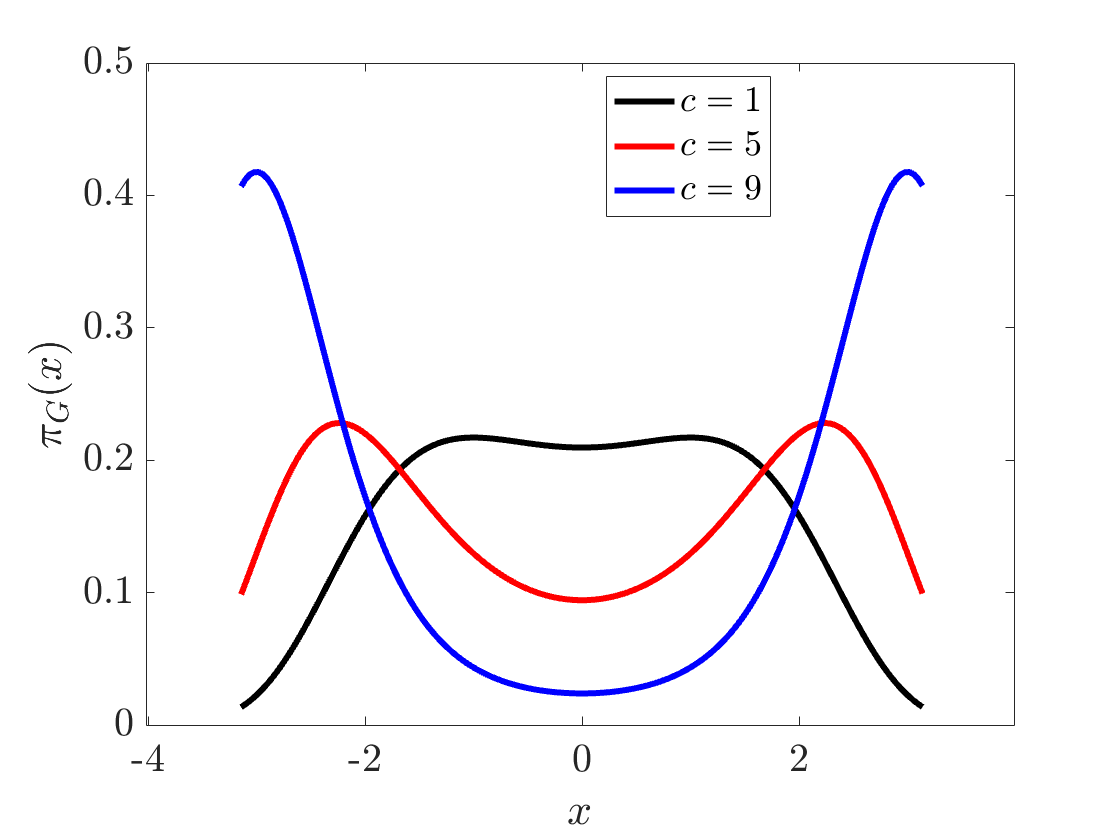}}
    \caption{(a) Double-well potentials considered in Section~\ref{subset:conv} and (b) their corresponding multimodal Gibbs distribution.}
    \label{fig:double-well-conv}
\end{figure}

\subsection{Double-well potential in one dimension}\label{subset:conv}
We consider a parameterized double-well potential 
\[
F(x) = \frac{1}{80}  \left( x^2 - c \right)^2 \,, \quad x\in [-\pi, \pi]\,,
\]
where $c$ is chosen from $\{1,5,9\}$. As shown in Figure~\ref{fig:double-well-conv}, $c$ has a modest effect on the energy gap $F_{\max} - F_{\min}$ but significantly influences the curvature of the potential.

We apply the Euler--Maruyama scheme and run $10{,}000$ i.i.d.~trajectories with the initial distribution $\mathcal{N}\left(-\frac{\pi}{2}, 0.01^2\right)$ over the time interval $[0, 20]$, using a time step of $\Delta t = 10^{-4}$. Figures~\ref{fig:double-well-conv-KL} and~\ref{fig:double-well-conv-chi2} display the $\text{KL}$ divergence and the $\chi^2$ divergence between $\rho_t$ and the target Gibbs distribution $\pi_G$ for each value of $c$. For the overdamped Langevin dynamics, the convergence rate decreases as $c$ increases, which results from a reduction in the minimum curvature of $F(x)$, i.e., $\min F''(x)$. According to Theorem~\ref{thm:langevin}, the case $c=9$ leads to a slower convergence rate compared to $c=1$. However, since $F_{\max} - F_{\min}$ is relatively insensitive to the value of $c$, the derivative-free dynamics exhibit a uniform convergence rate across all values of $c$. This observation aligns with our analysis in Theorem~\ref{thm:general_rate}, highlighting that the convergence of the derivative-free dynamics depends solely on the potential energy upper and lower bounds.

\begin{figure}
    \centering
\subfloat[KL divergence decay]{\includegraphics[width = 0.5\textwidth]{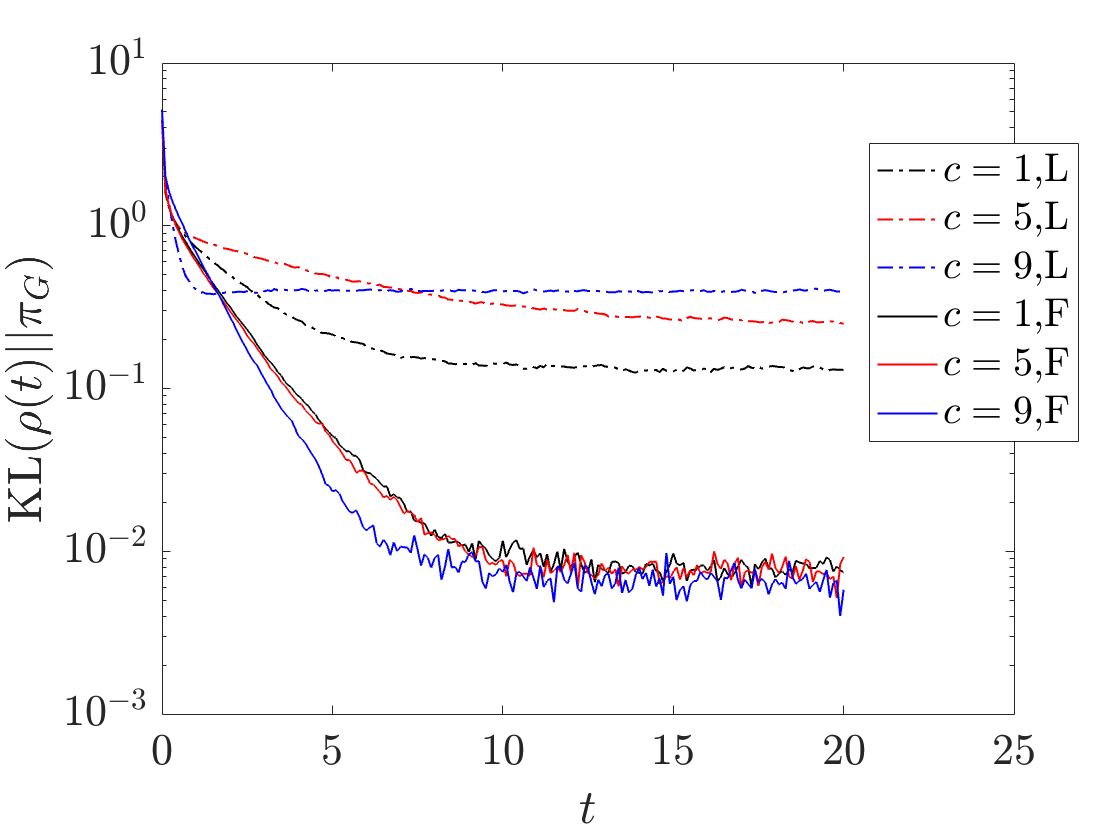}  \label{fig:double-well-conv-KL}}
\subfloat[$\chi^2$ divergence decay]{\includegraphics[width = 0.5\textwidth]{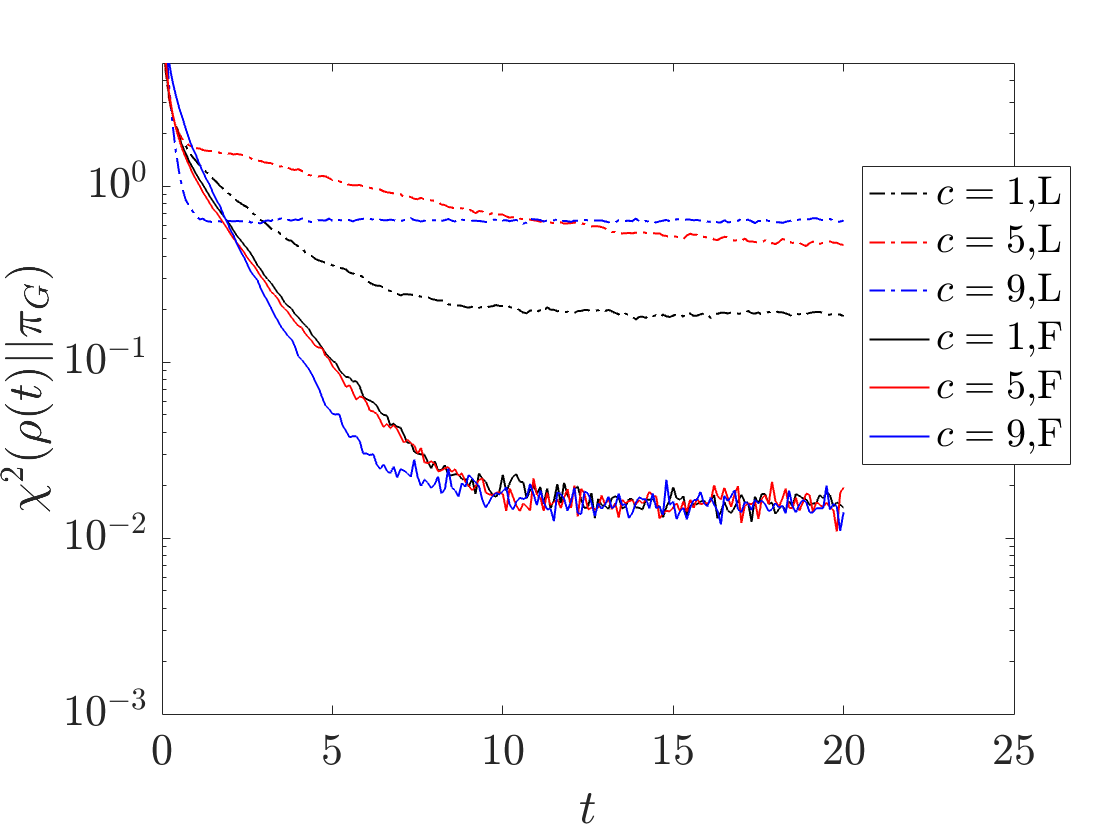} \label{fig:double-well-conv-chi2}}
    \caption{Convergence behavior of the KL divergence $\text{KL}(\rho_t||\pi_G)$ and the $\chi^2$ divergence $\chi^2(\rho_t||\pi_G)$  for the overdamped Langevin dynamics~\eqref{eq:overdamped Langevin} and  the derivative-free dynamics~\eqref{eq:diffuse_div_free}. As shown in Figure~\ref{fig:double-well-conv}, we consider a class of double-well potentials. The dashed lines are the results of the Langevin dynamics, and the solid lines are for the derivative-free dynamics. The cases of $c=1$, $c=5$, and $c=9$ are plotted in black, red, and blue, respectively.}
 
\end{figure}

\subsection{Multi-modal distribution in 2D}\label{subset:multimodal}

Next, we consider a multi-modal target distribution in 2D. Its Gibbs potential is
\[
F(\bx) = 2 |\sin(2\pi (\bx-0.1))|^2\,, 
\]
on the domain $[-1,1]^2 \subset \mathbb{R}^2$. The distribution has $16$ modes over the domain; see Figure~\ref{fig:2D potential} for an illustration. We are interested in sampling the Gibbs distribution $\pi_G(\bx) \varpropto \exp{(-20 F(\bx))}$ (see Figure~\ref{fig:2D true}). That is, $\eps = 0.05$.

We run both the derivative-free dynamics~\eqref{eq:diffuse_div_free} and the overdamped Langevin dynamics~\eqref{eq:overdamped Langevin} for the time interval $[0,T]$ with $T= 10$ and the time step size $\Delta t = 10^{-4}$. The initial distribution $\rho_0 = \mathcal{N}([-0.2,-0.2]^\top, 0.01^2 I)$ where $I$ is the identity matrix. Since the initial distribution is highly localized, the particles must climb all the ``hills'' in the Gibbs potential to sample the distribution supported over the entire domain. We ran the dynamics with $10^5$ i.i.d.~particles, whose trajectories are utilized to perform density estimation at time $T = 10$.

We illustrate the resulting distributions for the two dynamics in Figures~\ref{fig:2D free} and~\ref{fig:2D langevin}. We can see that the overdamped Langevin barely moved away from the initial distribution $\rho_0$ after time $T$, only exploring a few neighboring modes of the target distribution. In Section~\ref{sec:time},  we reviewed the classic theory for the mean exit time between local minima, which is exponential in $\frac{1}{\eps}$ and the ``height'' of the hill. 

In contrast, particle trajectories following the derivative-free dynamics rapidly explored all $16$ modes of the potential, converging to a distribution close to the true one, with only minor errors arising from the density estimation step. The key reason the derivative-free dynamics achieve fast convergence is that they are not constrained by the curvature of the potential, which significantly slows down the overdamped Langevin dynamics.

\begin{figure}
\centering
\subfloat[Gibbs potential]{\includegraphics[angle=0,width=0.49\textwidth]{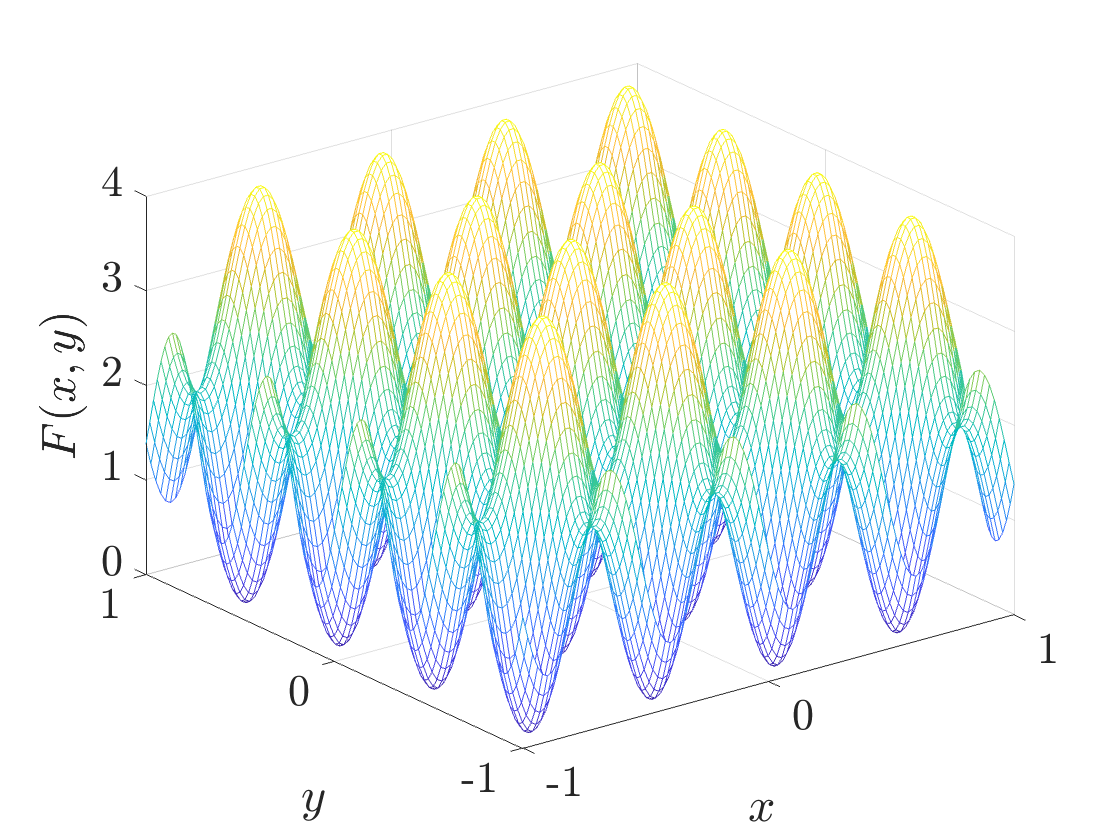}\label{fig:2D potential}}
\subfloat[True distribution]{\includegraphics[angle=0,width=0.49\textwidth]{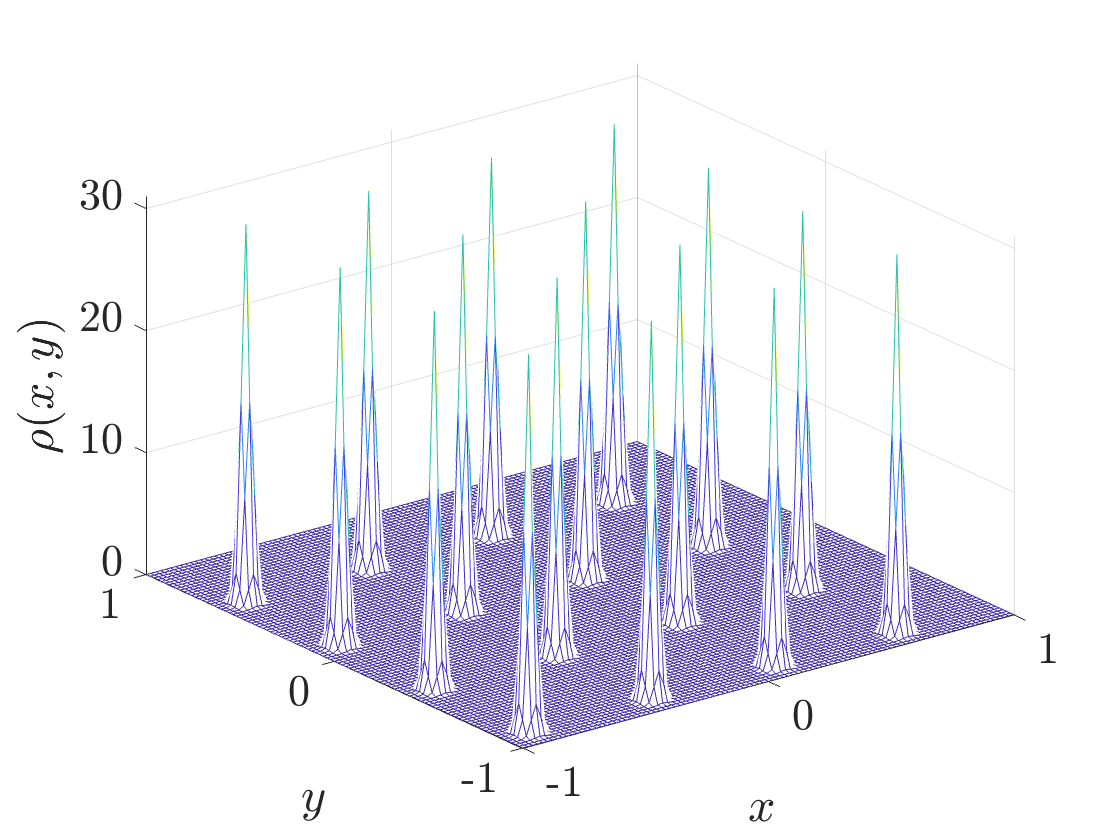}\label{fig:2D true}}\\
\subfloat[Derivative-free dynamics]{\includegraphics[angle=0,width=0.49\textwidth]{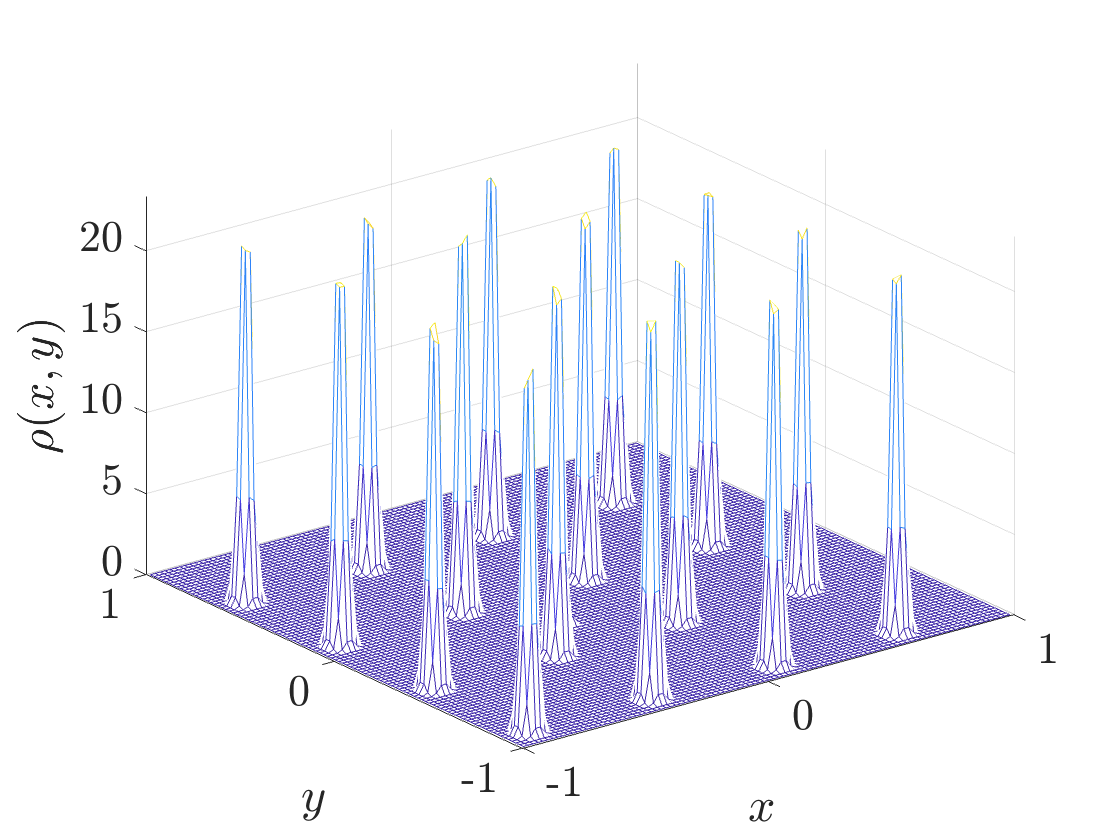}\label{fig:2D free}}
\subfloat[Langevin dynamics]{\includegraphics[angle=0,width=0.49\textwidth]{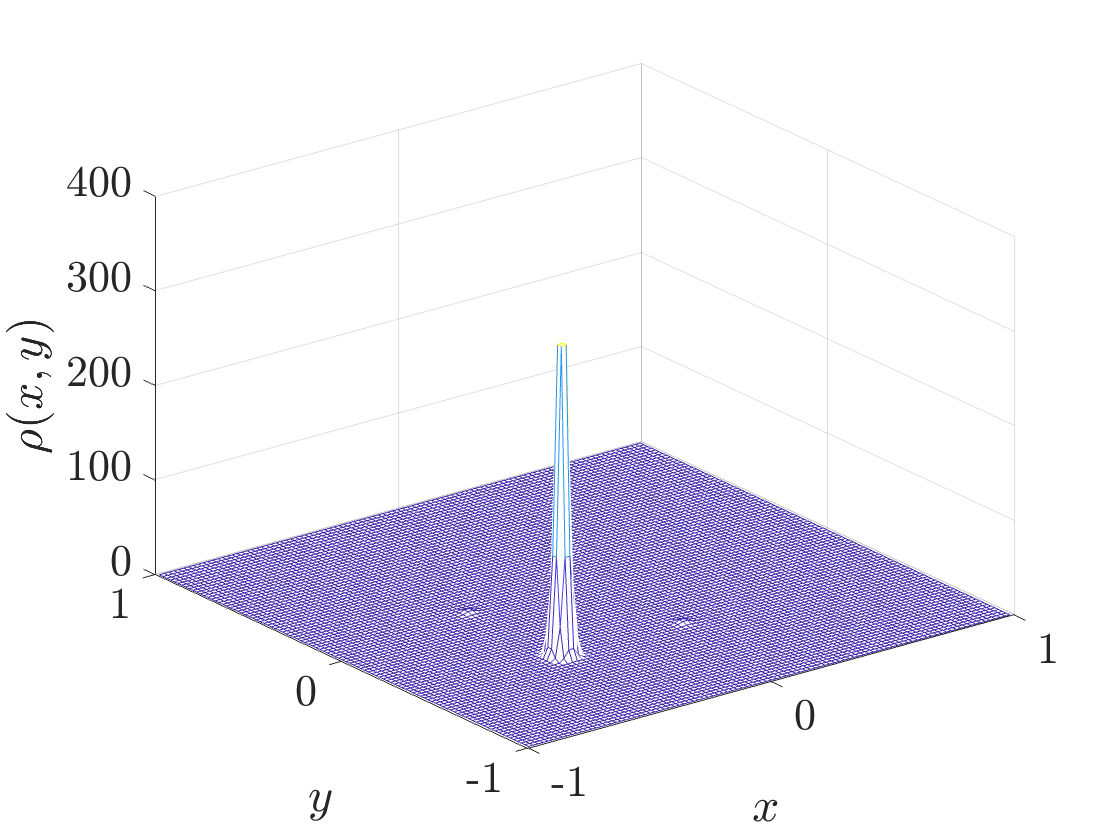}\label{fig:2D langevin}}
\caption{(a) The Gibbs potential $F$, (b) the target Gibbs distribution $\pi_G(\bx) \varpropto \exp{(-20 F(\bx))}$, (b) the probability distribution of the derivative-free dynamics at $T=10$, and (c) the probability distribution of the overdamped Langevin dynamics at $T=10$.\label{fig:2D conv}}
\end{figure}

\subsection{Mean exit  time}
In this subsection, we revisit the double-well example shown in Figure~\ref{fig:double-well} in Section~\ref{sec:time}. With the initial particle located at $m=0$, we are interested in the expected time it takes for the particle to reach either $a = -\pi$ or $b = \pi$. 

We test it on both the overdamped Langevin dynamics~\eqref{eq:overdamped Langevin} and the derivative-free dynamics~\eqref{eq:diffuse_div_free}. The mean exit times for both dynamics are shown in Figure~\ref{fig:time}. It is evident from the plots that $T = \mathcal{O}\left(\exp\left(\frac{1}{\eps}\right)\right)$ for the overdamped Langevin dynamics and is $\mathcal{O}\left(\frac{1}{\eps}\right)$ for the derivative-free dynamics. This matches our analysis in Section~\ref{sec:time}.

\begin{figure}[!htbp]
    \centering
    \subfloat[Overdamped Langevin dynamics~\eqref{eq:overdamped Langevin}]{\includegraphics[width=0.5\linewidth]{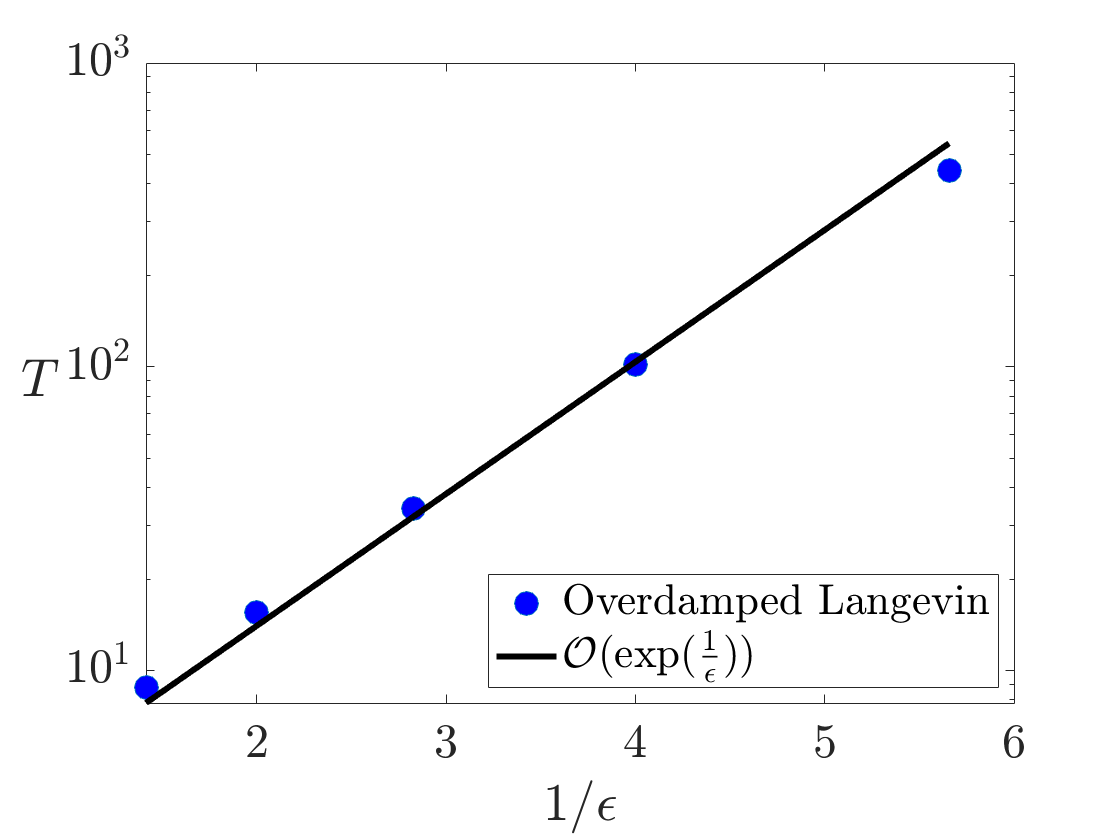}}
    \subfloat[Derivative-free dynamics~\eqref{eq:diffuse_div_free}]{\includegraphics[width=0.5\linewidth]{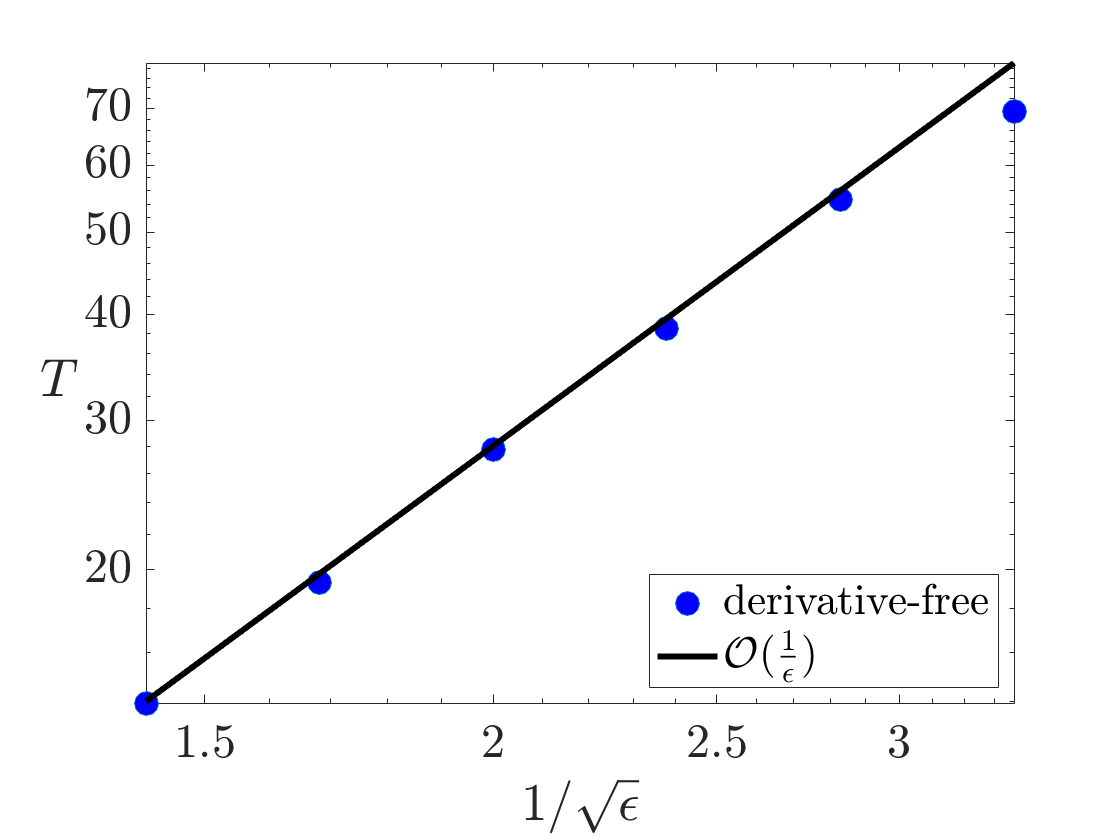}}
    \caption{Mean exit time comparison between the the overdamped Langevin dynamics~\eqref{eq:overdamped Langevin} and the derivative-free dynamics~\eqref{eq:diffuse_div_free} for the example shown in Figure~\ref{fig:double-well}. All dynamics start at $m=0$, and we estimate $T = \mathbb{E} \left[\inf_T \left(|X_T|-\pi|\right)^2 <10^{-4}\right]$. The expectations are estimated using $1000$ i.i.d.~runs. It is evident from the plots that $T = \mathcal{O}\left(\exp\left(\frac{1}{\eps}\right)\right)$ for the overdamped Langevin dynamics and is $\mathcal{O}\left(\frac{1}{\eps}\right)$ for the derivative-free dynamics.}
    \label{fig:time}
\end{figure}

\section{Concluding remarks}\label{sec:conclusion}

In conclusion, this work introduces and analyzes a class of Langevin-type Monte Carlo sampling algorithms with the novel use of a state-dependent adaptive diffusion term. The proposed method targets the Gibbs distribution on a bounded domain with nonconvex Gibbs potentials. Employing a novel state-dependent variance scheme enables efficient sampling even without requiring gradient information, making it particularly suitable for large-scale applications where gradient evaluation is costly. The study demonstrates that the derivative-free variant of the proposed algorithm can achieve faster convergence compared to traditional overdamped Langevin dynamics when dealing with complex distributions. The mean-exit time for the derivative-free dynamics does not depend on the energy gap between local minima and local maxima, but it is the case for overdamped Langevin dynamics, which results in slow mixing. Additionally, our work highlights the advantages of treating sampling algorithms as weighted Wasserstein gradient flows, and this unique perspective introduces new insights into the convergence behavior of these methods. The theoretical results are supported by numerical experiments, which illustrate the potential of this approach for efficient sampling in challenging high-dimensional settings.
\section*{Acknowledgment}
 
We would like to thank Dr.~Rafael Bailo (Eindhoven University of Technology) for pointing out the connection between the variable diffusion equation and the modified Wasserstein gradient flow. 
This work is partially supported by the National Science Foundation through grants DMS-2110895 (BE), DMS-1937254 (KR), DMS-2309802 (KR), and DMS-2409855 (YY), and by the Office of Naval Research  through grant
N00014-24-1-2088 (YY).


\appendix

\section{Weighted Poincar\'e Inequality}

One key component in our analysis is the weighted Poincar\'e inequality with a given weight function $w(\bx):\Omega \mapsto [0, \infty)$. To increase the readability of our proof, we recall here the inequality. The material here is standard and can be found in the references cited. For a general weight function $w(\bx)$,  we first introduce the Muckenhoupt
$A_p$ weights.
\begin{definition}[$A_p$ weights]\label{DEF:Ap}
For a fixed $1 < p < \infty$, we say that a weight function $w:\mathbb{R}^d \mapsto [0, \infty)$ belongs to the class $A_p$ if $w$ is locally integrable, and for all cubes $Q\subset \mathbb{R}^d$, we have
\begin{equation*}
[w]_p : = \sup_{Q\subset \mathbb{R}^d} \left({\frac {1}{V_Q }}\int _{Q}w (\bx)\,d\bx\right)\left({\frac {1}{V_Q}}\int _{Q}w (\bx)^{-{\frac {q}{p}}}\,d\bx\right)^{\frac {p}{q}} <\infty \,,
\end{equation*}
where $q$ is a real number such that $\frac{1}{p} + \frac{1}{q} = 1$, and $V_Q$ is the volume of the cube $Q$.
\end{definition}

The following weighted Poincar\'e inequality for weights in the $A_p$ class can be found in~\cite[Proposition~11.7]{ perez2019degenerate}.
\begin{theorem}[Weighted Poincar\'e inequality~\cite{perez2019degenerate}]\label{thm:wpi}
Let $w$ be an $A_p$ weight function and $f(\bx)$ a Lipschitz function. Then the following weighted Poincar\'e inequality holds for the hypercube $\Omega \subset \mathbb{R}^d$:
\begin{equation} \label{eq:wpi}
\frac{1}{w(\Omega)} \int_\Omega |f-f_{\Omega,w}|^p  w\, d\bx \leq\frac{ 2^p }{w(\Omega)} \int_\Omega |f-f_{\Omega}|^p  w\, dx \leq  \frac{C_d^p\, \ell_\Omega^p\, [w]_p\, }{w(\Omega)} \int_\Omega |\nabla f|^p w\, d\bx\,,
\end{equation} 
where $w(\Omega) = \dint_\Omega w(\bx) d\bx$,  $f_{\Omega,w} = \frac{1}{w(\Omega)} \dint_\Omega f(\bx) w(\bx) d\bx$, $f_{\Omega} = \frac{1}{V_\Omega} \dint_\Omega f(\bx) d\bx$, $\ell_\Omega$ is the side length of the cube $\Omega$, and $C_d$ is a dimensional constant.  
\end{theorem}

There have been many results on the weighted Poincar\'e inequality~\cite{fabes1982local,heinonen2018nonlinear}. The paper by P\'erez and Rela~\cite{perez2019degenerate} improved some of the classical results and produced a quantitative control of the Poincar\'e constant (see~\eqref{eq:wpi}) in the inequality, which is crucial for the analysis of our algorithm. We refer interested readers to~\cite{heinonen2018nonlinear,perez2019degenerate} for more general weighted Poincar\'e and Poincar\'e--Sobolev inequalities in various settings.

\begin{remark}
The definition of the $A_p$ class allows one to consider degenerate and singular weights. For example, let $w(\bx) = |\bx|^{\eta}$, $\bx\in\bbR^d$. Then $w\in A_p$ if and only if $-d< \eta < d(p-1)$.
\end{remark}

\end{document}